%% file: MAIN.tex
\documentclass{egpubl}

\usepackage{amsmath}

\usepackage{amsthm}
\usepackage{amsfonts}
\usepackage{amssymb}

\usepackage{egsgp2021}

\usepackage{xcolor} 

\JournalSubmission    
\usepackage[T1]{fontenc}
\usepackage{dfadobe}

\biberVersion
\BibtexOrBiblatex
\usepackage[backend=biber,bibstyle=EG,citestyle=alphabetic,backref=true]{biblatex}
\electronicVersion
\PrintedOrElectronic

\ifpdf \usepackage[pdftex]{graphicx} \pdfcompresslevel=9
\else \usepackage[dvips]{graphicx} \fi

\usepackage{egweblnk}

\DeclareFontEncoding{LS1}{}{}
\DeclareFontSubstitution{LS1}{stix}{m}{n}
\DeclareSymbolFont{symbols2}{LS1}{stixfrak} {m} {n}
\DeclareMathSymbol{\operp}{\mathbin}{symbols2}{"A8}

\DeclareMathOperator*{\LB}{{_{LB}}}
\DeclareMathOperator*{\DS}{{_{DS}}}

\usepackage{dsfont}
\usepackage{cancel}
\usepackage[normalem]{ulem}

\usepackage{pgfplots}
\usepackage{caption,subcaption}
\pgfplotsset{compat=1.12}

\usepackage{multirow}
\usepackage{booktabs}
\usepackage{arydshln}

\usepackage{tikz}
\usepackage{tikzscale}

\usepackage{graphicx}

\usepackage{overpic}

\pgfplotsset{every axis/.append style={
                    axis line style={-},
                    label style={font=\footnotesize},
                    tick label style={font=\footnotesize}
                    }}

\usepackage{amsmath}

\DeclareMathOperator*{\argmin}{arg\,min}

\newcommand{\tightparagraph}[1]{\vspace{-1mm}\paragraph*{#1}}

\makeatletter 
\newtheorem*{rep@theorem}{\rep@title}
\newcommand{\newreptheorem}[2]{%
\newenvironment{rep#1}[1]{%
 \def\rep@title{#2 \ref{##1}}%
 \begin{rep@theorem}}%
 {\end{rep@theorem}}}
\makeatother

\newtheorem{theorem}{Theorem}
\newreptheorem{theorem}{Theorem}

\newreptheorem{definition}{Definition}

\newreptheorem{proposition}{Proposition}
\newtheorem{lemma}[theorem]{Lemma}
\newreptheorem{lemma}{Lemma}

\usepackage{layouts}

\begin{document}

\newcommand{\inputFigPath}{./Figures/}%
\graphicspath{ {\inputFigPath} }





\title[Non-Isometric Shape Matching via Functional Maps]%
      {Non-Isometric Shape Matching via Functional Maps on Landmark-Adapted Bases}

\author[M. Panine \& M. Kirgo \& M. Ovsjanikov]
{\parbox{\textwidth}{\centering M. Panine $^{1, \text{~formerly~} 2}$ \orcid{0000-0001-7946-0584}
        , M. Kirgo $^{2,3}$\orcid{0000-0001-9154-3032} 
        and M. Ovsjanikov $^{2}$\orcid{0000-0002-5867-4046}
        }
        \\
{\parbox{\textwidth}{\centering $^1$Universit\`{a} della Svizzera italiana, Switzerland\\
         $^2$LIX, \'{E}cole Polytechnique, IP Paris, France\\
         $^3$ EDF R\&D, France
       }
}
}


\teaser{
\includegraphics[width=\linewidth]{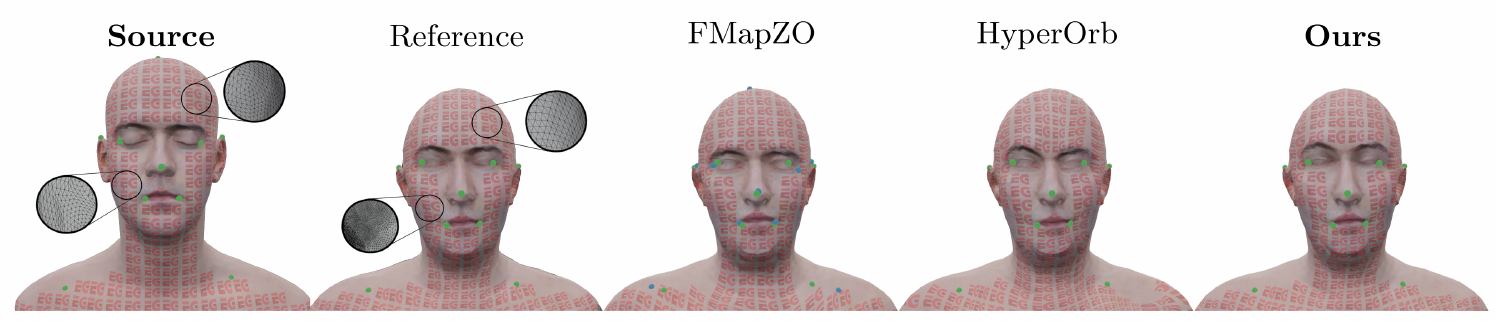}
\centering
\caption{\label{fig:teaser} Illustration of our method on a texture
  transfer problem, between two surfaces with significantly different
  mesh structure. The source model and its texture were produced by \cite{McGuire2017Data} (LPS Head) and the target model was extracted from the Faust dataset~\cite{bogo2014faust}. The user-specified landmark placement is shown in
  green, whereas the \emph{computed} landmarks are shown below in blue. The state-of-the-art functional maps-based method ``FMapZO'' \cite{melzi2019zoomout} fails to preserve landmarks exactly, whereas the hyperbolic orbifolds \cite{aigerman2015orbifold} (``HyperOrb'') approach leads to a map with higher distortion compared to our approach. The ``reference'' transfer was obtained using the commercial registration tool R3DS Wrap~\cite{R3DS} and $33$ user-defined landmarks.}
}

\maketitle

\begin{abstract}
We propose a principled approach for non-isometric landmark-preserving non-rigid shape matching. Our method is based on the functional maps framework, but rather than promoting isometries we focus instead on near-conformal maps that preserve landmarks exactly. We achieve this, first, by introducing a novel landmark-adapted basis using an intrinsic Dirichlet-Steklov eigenproblem. Second, we establish the functional decomposition of conformal maps expressed in this basis. Finally, we formulate a conformally-invariant energy that promotes high-quality landmark-preserving maps, and show how it can be solved via a variant of the recently proposed ZoomOut method that we extend to our setting. Our method is descriptor-free, efficient and robust to significant mesh variability. We evaluate our approach on a range of benchmark datasets and demonstrate state-of-the-art performance on non-isometric benchmarks and near state-of-the-art performance on isometric ones.

\end{abstract}


\input{Sections/s00_intro.tex}

\input{Sections/s01_RelatedWork.tex}

\input{Sections/s02_Overview.tex}

\input{Sections/s03_BackgroundNotation.tex}

\input{./Sections/s04_Theory.tex}

\input{Sections/s05_Energy.tex}

\input{./Sections/s06_SolvingProblem.tex}

\input{Sections/s07_Evaluation.tex}

\input{Sections/s08_Conclusion.tex}

\input{Sections/s09_Acknowledgements.tex}

\printbibliography

\appendix
\input{Sections/s10_Discretization}

\input{Sections/s11_DiscreteLandmarkCircle}

\input{Sections/s12_ProofOfLemmaOrthogonalDecomposition}

\input{Sections/s13_Appendix}

\end{document}

%% file: Sections/s00_intro.tex
\section{Introduction}
A common scenario in shape matching is that of very sparse user-provided landmark correspondences that need to be extended to a full map between the considered shapes. The landmarks in question are often of a semantic nature, and thus are very sensitive to exact placement. Consider, for instance the position of the eyes or the nose on a human face (see Fig. \ref{fig:teaser}) that are matched by an artist, e.g., in a texture transfer scenario. In such cases, it is crucial to preserve the landmark correspondences \emph{exactly} when extending the map. Furthermore, it is desirable for the extension process to be time-efficient and applicable to general, possibly non-isometric shape pairs.

Functional map methods \cite{ovsjanikov2017computing} constitute a highly effective shape matching framework, especially when coupled with powerful recent post-processing tools such as ZoomOut and its variants \cite{melzi2019zoomout,huang2020consistent}. The existing methods, however, suffer from two major limitations: first, they heavily rely upon the assumption of near-isometry, and second, they typically formulate landmark correspondence via descriptor preservation objectives, combined with other regularizers in the least squares sense. Unfortunately, this implies that the final map is not guaranteed to preserve user-provided landmark correspondences.

In this paper, we propose a novel approach that maintains the efficiency and flexibility of the functional maps pipeline, while overcoming these drawbacks. We organize our proposal in three major stages. First, we introduce a novel functional basis in which to express our map. Crucially, our basis is explicitly adapted to the landmark correspondences, unlike the commonly-used general Laplace-Beltrami eigenbasis. Intuitively speaking, this allows us to enforce landmark preservation by only considering functional maps with a particular (block-diagonal) structure. The design of this landmark-adapted basis is the most technically involved part of our proposal, and relies on solving intrinsic Dirichlet-Steklov and Dirichlet Laplacian eigenproblems. Specifically, we first construct new boundaries at the landmarks, and then formulate and solve the associated boundary value problems. 

Second, we remove the assumption of near-isometry by structuring shape matching as a search for bijective \emph{near-conformal} maps, which are significantly more general than isometries. Following the functional maps pipeline, we express this as a carefully designed energy to be minimized.

Third, we propose an iterative minimization strategy for our energy by following in the footsteps of ZoomOut \cite{melzi2019zoomout}. In particular, we demonstrate how landmark correspondences can be promoted throughout this iterative refinement. Furthermore, we exploit the landmark-awareness of our basis to provide a simple initial guess of the correspondence.

We test our approach on various benchmark datasets, both isometric and non-isometric. We compare our results to both state-of-the-art functional maps approaches, as well as recent methods that exactly preserve landmark correspondences. We report state-of-the-art accuracy on non-isometric datasets and near state of the art on isometric ones. Meanwhile, the computation time of our approach is significantly lower than that of the competing landmark-preserving methods.

\noindent \emph{Contributions.} To summarize: 
\begin{enumerate}
    \item We introduce a novel landmark-dependent functional basis by solving an intrinsic Dirichlet-Steklov eigenproblem.
    \item We formulate a functional maps-based approach to near-conformal shape matching that preserves given landmarks exactly without restrictions on the topology of the shapes.
    \item We propose an efficient way to both compute the basis and to solve the shape matching problem and report state-of-the-art results on difficult non-isometric benchmarks.
\end{enumerate}

%% file: Sections/s01_RelatedWork.tex
\section{Related Work}
\label{sec:Related}
Non-rigid shape matching is a well-established research area with a rich history of solutions. Below we review the works that are most closely related to ours, focusing on functional maps and landmark-preserving methods, and refer the interested readers to recent surveys \cite{van2011survey,sahilliouglu2020recent} for a more in-depth discussion.

\tightparagraph{Functional Maps}
Our approach fits within the functional maps framework that was originally introduced in \cite{ovsjanikov2012functional} and extended in many follow-up works, including including
\cite{kovnatsky2013coupled,aflalo2013spectral,rodola2017partial,ezuz2017deblurring,burghard2017embedding,nogneng17,melzi2019zoomout,MapTree} to name a few. An early overview of many functional maps-based techniques is given in \cite{ovsjanikov2017computing}.
The key idea exploited in all of these techniques is to represent correspondences as linear transformations across functional spaces, which can be compactly encoded as small-sized matrices given a choice of basis. This leads to simple optimization problems that can accommodate a range of geometric objectives such as isometry \cite{ovsjanikov2012functional}, accurate descriptor preservation \cite{nogneng17}, bijectivity \cite{eynard2016coupled}, orientation preservation \cite{ren2018continuous} or even partiality \cite{rodola2017partial} among others. Typically, such objectives are formulated as soft penalties on the functional map and are optimized for in the least squares sense.

\tightparagraph{Landmarks in functional maps} Landmark constraints are commonly used in functional maps-based approaches, especially in an attempt to resolve symmetry ambiguity, present, e.g., when mapping between human shapes.  Starting from the segment correspondences advocated in the original approach \cite{ovsjanikov2012functional}, and exploited in follow-up works, e.g., \cite{kleiman2018robust}, several techniques also used \emph{pointwise landmarks}, that were either user-specified \cite{nogneng17}, automatically computed \cite{FARM}, or even extended to \emph{curve constraints} \cite{gehre2018interactive}. All of these techniques, however, formulate landmark correspondences via functional descriptor preservation, e.g., based on the heat kernel \cite{sun2009concise,ovsjanikov2010} or wave kernel maps \cite{aubry2011wave}, which are enforced during optimization only in a least squares sense, alongside other descriptors and regularizers. Therefore, there is no guarantee that the final recovered point-to-point map will satisfy these user-constraints. In contrast, our approach is geared towards preserving the landmark correspondences \emph{exactly}, while computing a smooth overall map.

\tightparagraph{Landmark-based matching}
Landmark-preserving shape correspondence has also been studied in other matching frameworks. Early methods relied on extrinsic shape alignment, under given constraints, e.g., using thin plate splines \cite{bookstein1989principal,chui2000new} or by extending non-rigid ICP, as done in \cite{sumner2004deformation} among others. Such approaches, however, rely strongly on the shape embedding and often require a significant number of landmarks to work well in practice.

Another successful class of approaches have aimed to compute correspondences by embedding shapes to a common parametrization domain. This includes powerful approaches based on mapping surfaces to the planar domain, \cite{aigerman2014lifted,weber2014locally}, Euclidean orbifolds \cite{aigerman2015orbifold} general flat cone manifolds \cite{aigerman2015seamless} or, more recently, the hyperbolic plane \cite{tutte}, which can accomodate an arbitrary number of landmarks.

Finally, recent techniques have also allowed landmark-preserving shape correspondence by cross-parametrizing the surfaces directly. This includes exploiting direct and inverse \emph{averages} on surfaces \cite{panozzo2013weighted} or finding maps that minimize various notions of distortion, e.g., harmonicity and reversibility (using, first a surrogate high-dimensional embedding) \cite{ezuz2019reversible} or a related symmetric Dirichlet energy \cite{schmidt2020inter} (via direct optimization on the surface). These recent techniques can lead to accurate results, but are often computationally expensive, and typically place restrictions on the topology of the shape pair, such as having the same genus. In contrast, our method does not suffer from this limitation, as topological stability is one of the features of functional map methods, which is also inherited by our technique.

\tightparagraph{Basis selection for functional maps}
Finally, we remark that our construction of landmark-adapted functional bases also fits within the functional map framework, aimed at developing flexibile and effective basis functions. The original article and most follow-up works \cite{ovsjanikov2017computing} have advocated using the eigenfunctions of the Laplace-Beltrami operator (LBO), which are optimal for representing smooth functions with bounded variation \cite{AflBreKim15}. However, the Laplace-Beltrami basis has global support and may not be fully adapted to non-isometric shape changes. 

The compressed manifold modes~\cite{Neumann,ozolicnvs2013compressed, KovGlaBro16} have been introduced to offset the global nature of the LBO by promoting sparsity and locality in the basis construction. In a related effort, Choukroun \emph{et al.}~\cite{hamiltonian} have proposed to modify the LBO through a potential function, thus defining a Hamiltonian operator, whose eigenfunctions have better localization properties. In~\cite{LMH}, a similar approach was introduced to obtain a basis that is also orthogonal to a given set of functions. The ``Coordinate Manifold Harmonics'' used in ~\cite{CMH}, complement the LBO eigenfunctions with the coordinates of the 3D embedding, allowing to capture both extrinsic and intrinsic information. Finally, a rich family of diffusion and harmonic bases have been proposed in~\cite{patane2019}, by exploiting the properties of the heat kernel.

While these basis constructions offer more flexibility and have been shown to improve the functional map pipeline in certain cases, e.g., \cite{Neumann,CMH}, they nevertheless are typically still geared towards approximate isometries, and only enable approximate constraint satisfaction. In contrast our basis is geared towards landmark-preserving maps during functional map optimization, as well as during refinement.

\tightparagraph{Dirichlet-Steklov basis} Finally, we note that Steklov eigenproblems have been considered within geometry processing \cite{wang2018steklov} as tool for \emph{extrinsic} shape analysis. This is achieved by considering the (two-dimensional) surface as the Steklov boundary of its (three-dimensional) interior. In contrast, we consider a fully \emph{intrinsic} problem by using (one-dimensional) boundaries of small disks centered around the landmarks as the boundary of the remainder of the surface.

%% file: Sections/s02_Overview.tex
\section{Method Overview}
\label{sec:overview}
In this section, we provide a high-level overview of our approach. Our method takes as input a pair of shapes $\mathcal{M}$, $\mathcal{N}$ represented as triangle meshes along with two sets of $k$ landmark vertices $\{ \gamma^{_\mathcal{M}}_i \}_{i=1}^{k} \subset \mathcal{M}$, $\{\gamma^{_\mathcal{N}}_i \}_{i=1}^{k} \subset \mathcal{N}$. We then aim to compute a high-quality vertex-to-vertex correspondence $\varphi: \mathcal{N} \rightarrow \mathcal{M}$ that preserves the given landmarks exactly. I.e., $\varphi(\gamma^{_\mathcal{N}}_i) = \gamma^{_\mathcal{M}}_i$ for all $i$.

\begin{figure}[t!]
    \centering
    \includegraphics[width=\columnwidth]{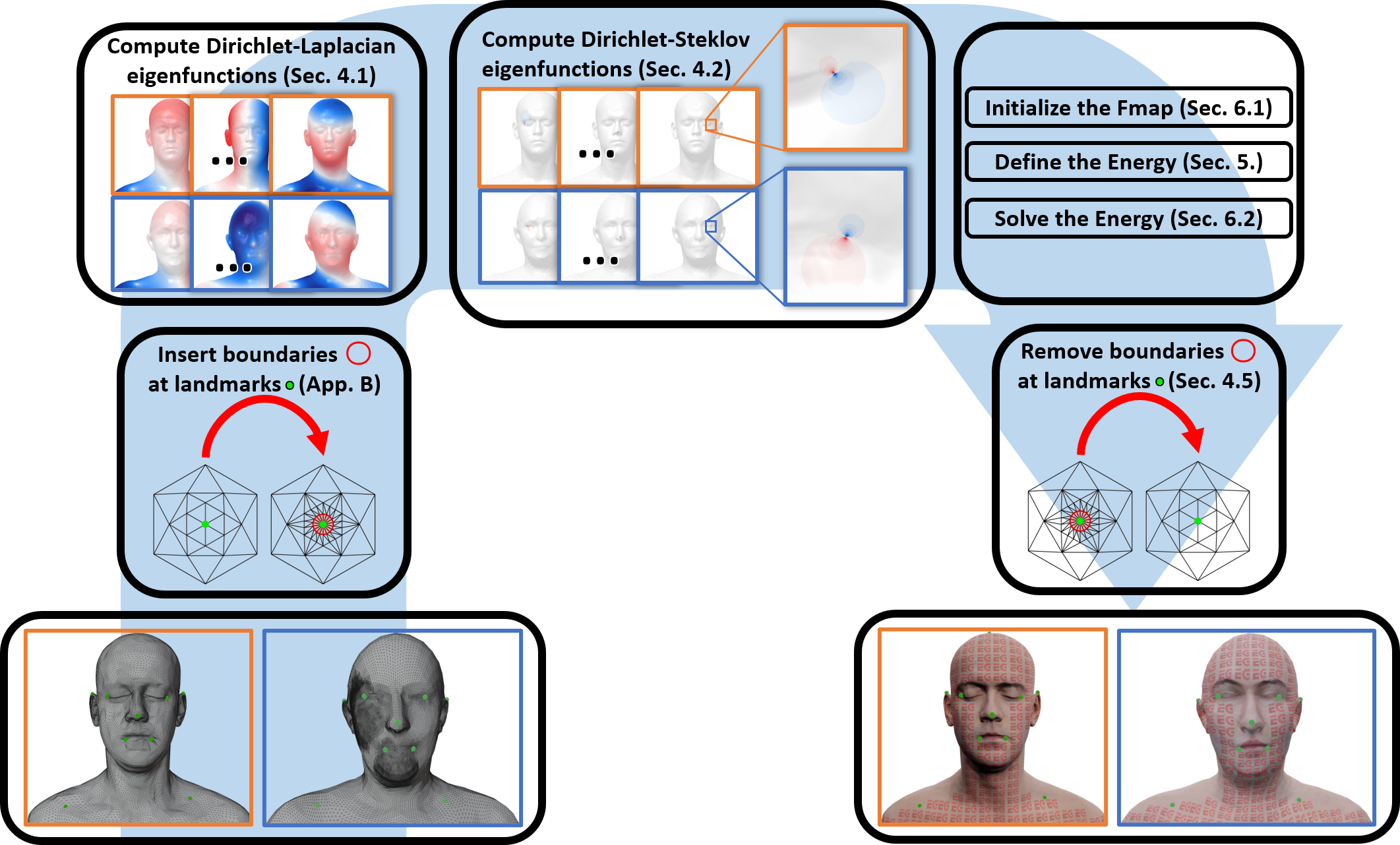}
    \caption{\label{fig:pipeline} Schematic of the main steps involved in our method to map a source shape (orange) to a target shape (blue) as described in Sec.~\ref{sec:overview}.}
\end{figure}

Our overall strategy, illustrated in Fig.~\ref{fig:pipeline}, consists of the following major steps:
\begin{enumerate}
  \item Remove small disks from the mesh surface, centered at each landmark point $\gamma_i$, $i = 1 ... k$. This creates $k$ circular boundary components $\{\Gamma_i \}_{i=1}^k$, which are fully contained in the one-ring neighborhood of each landmark.
  \item Compute the set of the first $N_{\LB}$ Laplace-Beltrami eigenfunctions $\{\psi_j\}_{j=1}^{N_{\LB}}$ with \emph{Dirichlet boundary conditions} at the boundaries of landmark circles. I.e., $\psi_j|_{\Gamma_{i}} = 0$ for all $i,j$.
  \item Add to this basis another \emph{$k$ sets} of $N_{\DS}$ basis functions each $\{u^{(i)}_{j}\}_{j=1}^{{N_{\DS}}}$, $i=1...k$, one for each landmark circle, consisting of eigenfunctions of the intrinsic Dirichlet-Steklov eigenproblem. Each of these basis sets is well-suited to describing smooth functions in the vicinity of its corresponding landmark circle. Intuitively these functions complement the Laplace-Beltrami eigenbasis, are harmonic on the interior of the shapes, and are zero at all but one disk boundary: $u^{(i)}_j|_{\Gamma_{l}} = 0$ for $l \ne i$ and all $j$.
  \item Compute an optimal functional map by minimizing an energy that promotes near-conformal maps, via an iterative refinement strategy. We split the functional map into $k+1$ parts, and separately align the Laplacian eigenbasis and each set of Dirichlet-Steklov ones.
  \item Convert the computed functional map to a vertex-to-vertex map between the shapes with the disks cut out.
  \item Reinsert the landmark correspondences to obtain a landmark-preserving vertex-to-vertex map between the original meshes.
\end{enumerate}

Our general strategy follows the standard functional map pipeline, especially in its recent variants based on iterative refinement \cite{melzi2019zoomout,pai2021fast,xiang2021dual}, with several crucial changes.

First, our main motivation for introducing disks to represent landmarks in Step 1 is to associate to each landmark a well-defined functional space. In this, we are inspired by techniques that represent landmarks or seed points on a surface via associated harmonic functions \cite{zayer2005harmonic,patane2019} on a mesh. Unlike such harmonic functions, however, our construction is fully justified in the smooth setting. This is because it is impossible to impose boundary conditions on isolated points on a smooth surface. Furthermore, as we show below, the Dirichlet-Steklov eigenfunctions that we compute in Step~3. are \textit{orthogonal to the Dirichlet Laplacian eigenbasis} and jointly form a \emph{complete} basis for the underlying functional space.

Secondly, instead of computing a single functional map across Laplace-Beltrami eigenfunctions, we estimate a \textit{block-diagonal} functional map that aligns each of the $k+1$ components of the functional space separately. This both improves efficiency and promotes desirable structural map properties. Indeed, we prove that this splitting must hold for conformal maps in the smooth setting, and  we observe that it promotes preservation of landmark neighborhoods across a wide range of shape deformations in practice.

Finally, rather than focusing on near-isometries, we build a functional map energy that aims at computing near-conformal maps, and fully avoids the use of descriptor functions. Furthermore, we propose an efficient initialization and an iterative strategy for optimizing this energy, while promoting desirable map properties. This ensures high-quality correspondences even in challenging cases, in which existing functional maps-based methods tend to fail.

In the following sections, we discuss each step  of this pipeline in detail. Throughout our discussion related to the basis construction and the structure of pointwise and functional maps, we focus mainly on derivations in the smooth setting, to highlight the theoretically justified nature of our approach.

%% file: Sections/s03_BackgroundNotation.tex
\begin{figure}
    \includegraphics[width=\linewidth]{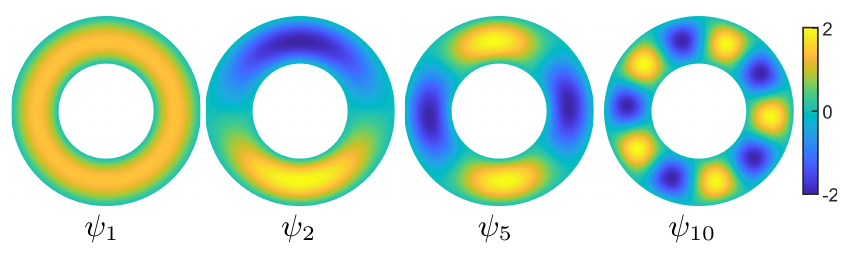}
    \caption{Several Dirichlet Laplacian eigenfunctions of the annulus in 2D with external radius $1$ and internal radius $1/2$. Notice that the eigenfunctions concentrate away from the boundary.\vspace{-2mm}}
    \label{Fig:AnnulusLBeigs}
\end{figure}

\section{Functional Basis} \label{sec:FunctionalBasis}

Central to our proposal is a careful choice of functional basis for a convenient functional space over the considered shapes. Our basis is built as the union of the solutions to the Dirichlet Laplacian and Dirichlet-Steklov eigenproblems, which we describe in Secs.~\ref{sec:LBeigenproblem} and~\ref{Sec:DirichletSteklov}, respectively. In Sec.~\ref{Sec:FunctionalSpace} we define the functional space that we use in the rest of the proposal. The constructions described in these sections pertain to manifolds with boundaries and are not yet specialized to our shape matching method, which can be used both for shapes with and without boundaries. The specialization to our case is carried out in Sec.~\ref{Sec:LandmarkAdaptedBasis}. There, we describe how to create a landmark adapted functional basis by, roughly speaking, treating the landmarks as boundaries. All the constructions discussed in this section are carried out in the smooth setting. Their discretization is discussed in App.~\ref{sec:DiscretizationEigenproblems}.

\subsection{Dirichlet Laplacian Eigenproblem} \label{sec:LBeigenproblem}

Let $\mathcal{M}$ be a smooth, connected, oriented compact Riemannian manifold with metric $g$ and a boundary $\partial \mathcal{M}$. The Dirichlet Laplacian eigenproblem is:
\vspace{-1mm}
\begin{equation} \label{Eq:DirichletLaplacian}
\begin{aligned}
	&\Delta \psi_i = \lambda_i \psi_i\\
	&\psi_i \big|_{_{\partial \mathcal{M}}} = 0~,
\end{aligned}
\end{equation}
\noindent where $\Delta$ denotes the non-negative Laplace-Beltrami operator. Despite the vanishing boundary condition, it can be shown (see \cite{chavel1984book}) that the eigenfunctions $\{ \psi_i \}_{i=1}^{\infty}$ can be chosen to form an orthonormal basis for $L_2(\mathcal{M})$ (square integrable functions of $\mathcal{M}$). Moreover, the eigenvalues and eigenfunctions can be ordered such that $0 < \lambda_1 \leq \lambda_2 \leq ... \to \infty$.

\begin{figure}
    \centering
     \includegraphics[width=\linewidth]{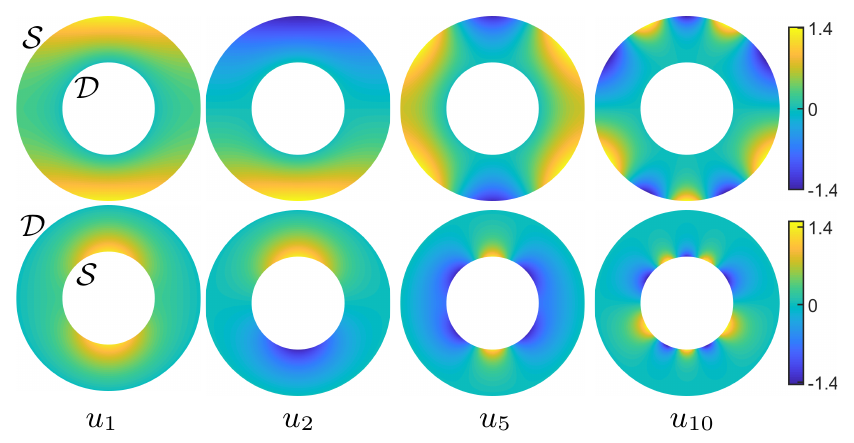}
    \caption{Some Dirichlet-Steklov eigenfunctions of the same annulus from Fig.~ \ref{Fig:AnnulusLBeigs}. The Steklov boundary condition is imposed in turn on the external (top row) and internal (bottom row) boundaries. Notice that the eigenfunctions concentrate on the Steklov boundary.\vspace{-2mm}}
    \label{Fig:AnnulusDSeigs}
\end{figure}

In Fig.~\ref{Fig:AnnulusLBeigs}, we illustrate the first few Dirichlet Laplacian eigenfunctions for an annulus in the plane with external radius $1$ and internal radius $1/2$. We will return to this example of the annulus in our discussion to compare the properties of different bases that we consider. We very briefly discuss the discretization of the Dirichlet Laplacian problem on triangle meshes in App.~\ref{sec:DiscretizationEigenproblems}.

\subsection{Dirichlet-Steklov Eigenproblem} \label{Sec:DirichletSteklov}

Let $\mathcal{M}$ be a smooth, connected, oriented compact Riemannian manifold with metric $g$ and a Lipschitz continuous boundary $\partial \mathcal{M}$. Suppose that, up to sets of measure $0$, $\partial \mathcal{M}$ consists of two disjoint nonempty open sets, denoted $\mathcal{D}$ and $\mathcal{S}$. The (mixed) Dirichlet-Steklov eigenproblem is posed as follows:
\begin{equation} \label{Eq:DirichletSteklov}
\begin{aligned}
	&\Delta u_i = 0\\
	&u_i \big|_{_{\mathcal{D}}} = 0\\
	&\partial_n u_i \big|_{_{\mathcal{S}}} = \sigma_{i} u_i~,
\end{aligned}
\end{equation}
where $\partial_n$ denotes the exterior normal derivative. The second and third line of the above are respectively known as the Dirichlet and Steklov boundary conditions, explaining the name Dirichlet-Steklov.

As hinted at in Sec.~\ref{sec:overview} and explained in detail in Sec.~\ref{Sec:LandmarkAdaptedBasis}, in our approach,
$\mathcal{S}$ will be the boundary corresponding to a given landmark, while $\mathcal{D}$ will be the union of all other landmark boundaries.

The general theory of the Dirichlet-Steklov and many other similar problems can be found in \cite{necas2012direct}. For a gentle introduction to Steklov eigenproblems, see \cite{labrie2017theoreme} (in French).

The  eigenvalues $\{\sigma_i\}_{i=1}^\infty$ can be ordered such that $0 < \sigma_1 \leq \sigma_2 \leq ... \to \infty$. Unlike the eigenfunctions of the Laplacian eigenproblem, the Dirichlet-Steklov eigenfunctions do not form an orthonormal basis for $L_2(\mathcal{M})$. Instead, the restriction of $\{u_i\}_{i=1}^{\infty}$ to the boundary $\mathcal{S}$ form such a basis for $L_2(\mathcal{S})$ (see \cite{necas2012direct}).

We emphasize that, in contrast to previous uses of the Steklov eigenproblem in \cite{wang2018steklov}, we consider a purely \emph{intrinsic} problem on the shape surface. I.e., as described in detail in Sec.~\ref{Sec:LandmarkAdaptedBasis} our boundaries are one-dimensional, being the boundaries of disks centered at the landmarks.

As it is written above, the Dirichlet-Steklov problem may seem a bit mysterious. However, it becomes much more approachable when written in weak form:
\begin{equation}\label{Eq:WeakDirichletSteklov}
	\int_{\mathcal{M}} \nabla f \cdot  \nabla u_i  ~d\mathcal{M} = \sigma_i \int_{\mathcal{S}} f u_i ~d(\partial \mathcal{M})~.
\end{equation}
In this form, it can be compared to the standard Laplacian eigenproblem: $\int_{_{\mathcal{M}}} \nabla f \cdot  \nabla \psi_i  ~d\mathcal{M} = \lambda_i \int_{_{\mathcal{M}}} f u_i ~d\mathcal{M}.$ Intuitively, and as we demonstrate in practice, the Dirichlet-Steklov eigenfunctions ``focus'' on the boundary $\mathcal{S}$ and provide detailed information in the vicinity of this boundary. As discussed below, in our method we establish one set of Dirichlet-Steklov eigenfunctions for each landmark, and align those functional spaces across the pair of shapes.

A derivation of the weak form of the Dirichlet-Steklov problem is provided in App.~\ref{sec:DSWeakForm}. The discretization of the resulting problem on triangle meshes is discussed in App. ~\ref{sec:DiscretizationEigenproblems}.

We illustrate some Dirichlet-Steklov eigenfunctions for the annulus in Fig. \ref{Fig:AnnulusDSeigs}. Notice that the eigenfunctions concentrate on the boundary on which the Steklov boundary condition is imposed.

%% file: Sections/s04_Theory.tex
\subsection{Functional Space $W(\mathcal{M})$} \label{Sec:FunctionalSpace}
In this section, we specify the functional space used for the remainder of our proposal. Recall that our goal is to obtain a variant of the functional maps method suitable for non-isometric shape matching. We propose to search for near-conformal maps.

We thus need to translate the search for near-conformality to the functional maps setting. We do so by building upon the foundations laid in the context of conformal shape differences \cite{rustamov2013map,corman2017functional}. Given a pair of surfaces $\mathcal{M}$ and $\mathcal{N}$, in \cite{rustamov2013map}, it is observed that to study the deviation from conformality of a map $\varphi: \mathcal{N} \to \mathcal{M}$, it is useful to consider its pullback $F_{_{\mathcal{MN}}}$ as a map between spaces of functions equipped with the Dirichlet form:
\begin{equation}
\langle f,u \rangle_{_{W(\mathcal{M})}} = \int_{\mathcal{M}} \nabla f \cdot \nabla u ~d\mathcal{M}~.
\end{equation}

The Dirichlet form becomes an inner product on the space of smooth functions modulo constant functions. A Hilbert space is then obtained by taking the completion of the space in the induced topology. We denote the space thus obtained by $W(\mathcal{M})$. We remark that this space is different from the standard $L_2$ space of square integrable functions, due to the additional \emph{smoothness} conditions. Below we describe both the properties and the utility of this space in the context of our landmark-based shape matching approach.

\subsection{Landmark-Adapted Basis for $W(\mathcal{M})$} \label{Sec:LandmarkAdaptedBasis}

As highlighted above, a key aspect of our approach is the construction of a novel functional basis that is \emph{adapted to the landmarks}.

Our main idea is to treat the landmarks as boundaries at which the functional bases satisfy certain boundary conditions. For this, we first slightly modify the shapes under study. Indeed, while advocated in several prior works \cite{zayer2005harmonic,patane2019} in geometry processing, it is not strictly speaking possible to impose boundary conditions at isolated points in the smooth setting.

Let $\mathcal{M}$ and $\mathcal{N}$ be compact, connected, oriented Riemannian surfaces. For simplicity, we also temporarily assume them to be without boundary. This last assumption is removed later. Let $\{\gamma_i^{_{\mathcal{M}}}\}_{i=1}^{k} \subset \mathcal{M}$ and $\{\gamma_i^{_{\mathcal{N}}}\}_{i=1}^{k} \subset \mathcal{N}$ be $k$ landmarks in one-to-one correspondence. That is, we will be looking for maps $\varphi: \mathcal{N} \to \mathcal{M}$ such that $\varphi(\gamma_i^{_{\mathcal{N}}}) = \gamma_i^{_{\mathcal{M}}}$ for all $i = 1  ... k$. Such $\varphi$ are said to be landmark preserving.
The functional map representation of $\varphi$, that is its pullback on functions, will be denoted $F_{_{\mathcal{MN}}}$, as before.

Our first step is to convert the landmarks into proper boundaries. We do so by removing small disks centered at the landmarks and treat the boundaries of these disks as boundaries of the shapes. We make sure that none of the disks intersect. Thus, we end up with a new shape that has $k$ boundary components, one for each landmark. We denote the boundary corresponding to the landmark $\gamma_i^{_{\mathcal{M}}}$ by $\Gamma_i^{_{\mathcal{M}}}$. By abuse of notation, we denote the shapes thus modified by $\mathcal{M}$ and $\mathcal{N}$, same as their original versions. On triangle meshes, we create boundaries that are fully contained in a one-ring neighborhood of each landmark. This operation is described in detail in App.~\ref{sec:DiscreteLandmarkCircle}.

We now use the newly created boundaries to split $W(\mathcal{M})$ into convenient subspaces. These subspaces will be composed of functions satisfying carefully chosen eigenvalue problems and boundary conditions.

We begin by considering the span of Laplace-Beltrami eigenfunctions satisfying Dirichlet boundary conditions on the $\{\Gamma_i\}_i$:
\begin{equation} \label{Eq:DirichletLaplaceBeltrami}
\begin{aligned}
	&\Delta \psi_i = \lambda_i \psi_i~,\\
	&\psi_{i}\big|_{_{\Gamma_j}} = 0~~,~~\forall i,j .
\end{aligned}
\end{equation}

Recall that the eigenfunctions $\{\psi_i\}_{i=1}^{\infty}$ form a orthonormal basis for $L_2(\mathcal{M})$. They remain mutually orthogonal in $W(\mathcal{M})$, but interestingly fail to form a full basis for that space. This counter-intuitive behavior is due to the change of topology from $L_2(\mathcal{M})$ to $W(\mathcal{M})$ and the infinite dimensionality of the functional spaces under consideration.

Let the $W(\mathcal{M})$ closure of the subspace spanned by the $\{\psi_i\}_{i=1}^{\infty}$ be denoted by $\mathcal{G}(\mathcal{M})$.

Naturally, our next step is to find functions that span the remainder of $W(\mathcal{M})$. This is where the Dirichlet-Steklov eigenfunctions of Sec.~\ref{Sec:DirichletSteklov} come in. We pose $k$ Dirichlet-Steklov problems, with the $j^{th}$ problem being:
\begin{equation} \label{Eq:DirichletSteklovOnLandmarks}
\begin{aligned}
	&\Delta u_i^{(j)} &&= &&0 && &&\\
	&u_i^{(j)} \big|_{_{\Gamma_q}} &&= &&0,   &&   &&q \neq j\\
	&\partial_n u_i^{(j)} \big|_{_{\Gamma_j}} &&= &&\sigma_{i}^{(j)} u_i^{(j)} &&.  &&
\end{aligned}
\end{equation}

\noindent This results in $k$ Dirichlet-Steklov eigenbases and spectra denoted $\{u_i^{(j)}\}_{i=1}^{\infty}$ and $\{\sigma_i^{(j)}\}_{i=1}^{\infty}$, respectively.
Recall that the $\{u_i^{(j)}\}_{i=1}^{\infty}$ form an orthonormal basis for $L_2(\Gamma_j)$. These functions remain mutually orthogonal in $W(\mathcal{M})$. This follows directly from the weak form of the Dirichlet-Steklov problem (Eq. \eqref{Eq:WeakDirichletSteklov}).
We denote the $W(\mathcal{M})$ closed span of the $\{u_i^{(j)}\}_{i=1}^{\infty}$ by $\mathcal{H}_{j}(\mathcal{M})$.

\indent Our key result is that, once put together, the Dirichlet Laplacian eigenfunctions and the $k$ sets of Dirichlet-Steklov eigenfunctions span all of $W(\mathcal{M})$.

\begin{lemma} \label{Lemma:WorthogonalSplit}
The function space $W(\mathcal{M})$ admits the following decomposition:
\begin{equation}
W(\mathcal{M}) = \mathcal{G(M)} \operp \overline{\left( \bigoplus_{j=1}^{k} \mathcal{H}_j (\mathcal{M}) \right)} ~,
\end{equation}
where $\oplus$ denotes direct sums and $\operp$ denotes orthogonal direct sums, and the overline denotes the  $W(\mathcal{M})$ closure of the spanned functional space.
\end{lemma}
\begin{proof}
See App.~\ref{Sec:AppendixProofOfSplit}.
\end{proof}

Intuitively, the above lemma says that $W(\mathcal{M})$ can be split into a non-harmonic part and $k$ harmonic landmark-associated subspaces, with each landmark getting its own subspace of harmonic functions that are nonvanishing on it.  In practice, we always $W(\mathcal{M})$-normalize all of the considered eigenfunctions by dividing each function by its $W$ norm. In all of the following, we use $W(\mathcal{M})$-normalized bases.

The resultant basis is thus normalized. However, it is not quite $W(\mathcal{M})$-orthogonal, as suggested by the notation used Lemma \ref{Lemma:WorthogonalSplit}. Specifically, the problem lies in the mutual non-orthogonality of the subspaces $\mathcal{H}_{j}(\mathcal{M})$. This is discussed in App.~\ref{Sec:AppendixProofOfSplit}.

In principle, the energy that we are to minimize (see Sec.~\ref{sec:EnergyFormulation} below) can be expressed in any basis, even if it is not orthogonal. For our purposes, however, the non-orthogonality of our basis poses a few challenges, which will be detailed later. Fortunately, in practice, \emph{our basis can be accurately approximated as orthonormal}. A typical matrix of $W(\mathcal{M})$ inner products is shown in Fig.~\ref{Fig:SphereWproducts} (see App.~\ref{Sec:appendix_near_orthogonality} for an extended evaluation of this approximation). In  Fig.~\ref{Fig:AnnulusWproducts} we evaluate the orthonormality in the case of the 2D annulus and observe that it becomes more and more valid as the radius of the inner disk becomes smaller. We will call attention to this approximation when we use it in the implementation of our proposal.

\begin{figure}
    \centering
    \includegraphics[width=\linewidth]{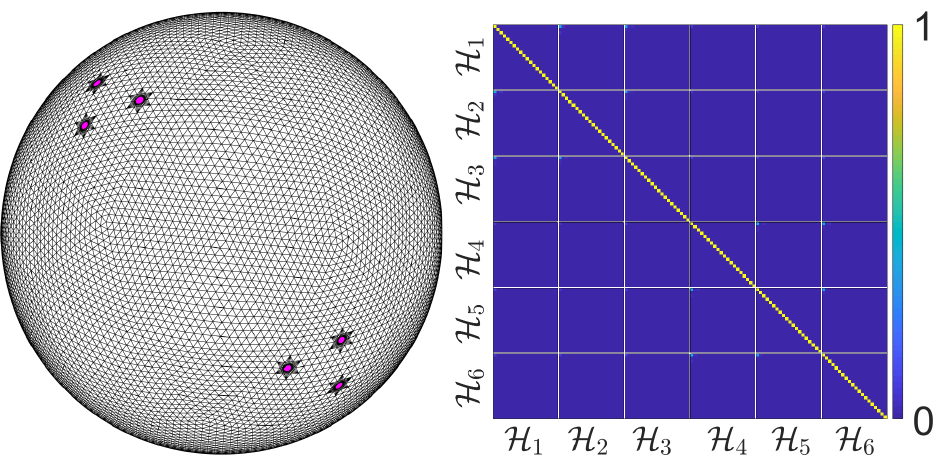}
    \caption{$W$-inner products for the first $20$ Dirichlet-Steklov eigenfunctions corresponding to six landmark circles on a sphere mesh (left). The first three landmarks are on the top left and the remaining three are on the bottom right. Notice that the different $\mathcal{H}_i$ subspaces are almost orthogonal.}
    \label{Fig:SphereWproducts}
\end{figure}

\begin{figure}
    \centering
    \begin{overpic}[width=\linewidth, trim=0 -20 0 0, clip]{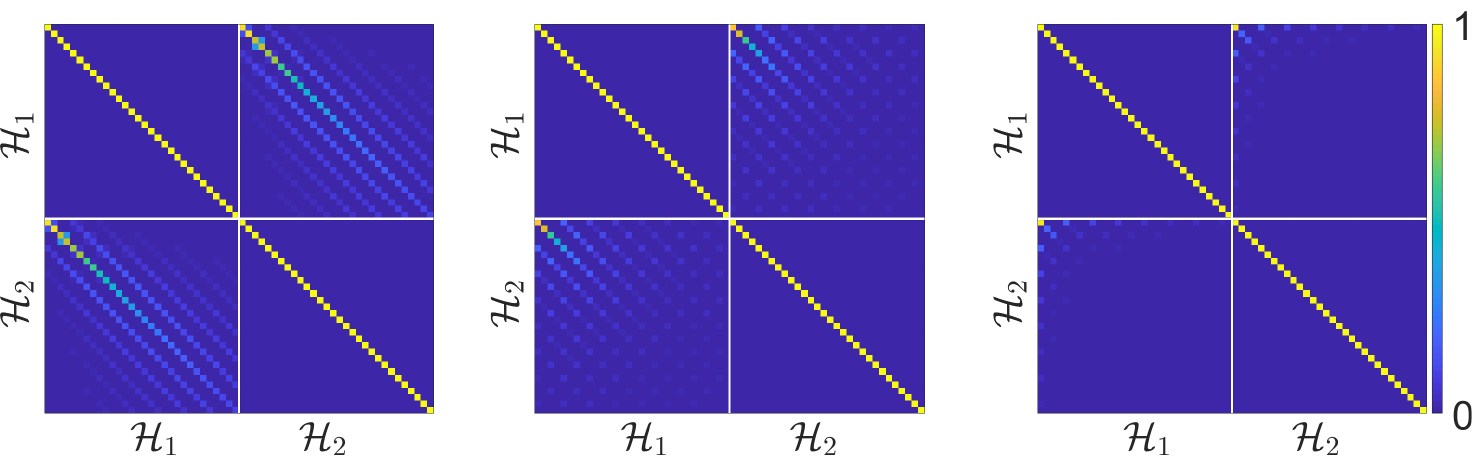}%
    \put(10,0) {$r_i = 0.8$}%
    \put(43,0) {$r_i = 0.5$}%
    \put(77,0) {$r_i = 0.1$}%
    \end{overpic}
    \caption{$W$-inner products for the first $30$ Dirichlet-Steklov eigenfunctions corresponding to the two boundaries of the annulus. The external radius of the annulus is $1$, while different values of the internal radius $r_i$ are considered. $\mathcal{H}_1$ and $\mathcal{H}_2$ correspond to the internal and external boundaries of the annulus, respectively. In this case, approximation of orthogonality of $\mathcal{H}_1$ and $\mathcal{H}_2$ fails for a large $r_i$, but becomes more and more valid as $r_{i}$ decreases.\vspace{-2mm}}
    \label{Fig:AnnulusWproducts}
\end{figure}

Before proceeding further, we note briefly that on shapes with pre-existing boundaries we impose Neumann boundary conditions (vanishing normal derivatives). The above discussion remains unchanged. Note that imposing Neumann boundary conditions requires no special effort in the discrete setting.

We now illustrate our functional basis using the landmark circles and Neumann boundary conditions on both the inner and outer boundary of the annulus in Fig. \ref{Fig:AnnulusBasisComparison}. Notice that as their eigenvalue increases, the Dirichlet-Steklov eigenfunctions rapidly concentrate on the landmark circles. In fact, the eigenfunctions of the closely related Steklov eigenproblem (i.e. without the Dirichlet boundary) are known to decay exponentially with distance from the Steklov boundary, the rate of decay being proportional to the corresponding eigenvalue \cite{polterovich2019nodal}. In contrast, the Dirichlet-Laplacian eigenfunctions remain evenly spread in the bulk of the shape. Thus, high eigenvalue Dirichlet-Steklov eigenfunctions are uninformative regarding the bulk of the manifold. Meanwhile, the high eigenvalue Dirichlet Laplacian eigenfunctions remain informative in the bulk even at high eigenvalues.

So far, we assumed that the considered shapes were connected. Our discussion remains unchanged on general shapes, as long as each connected component has at least two landmarks on it, as this is necessary to impose both boundary conditions of the Dirichlet-Steklov eigenproblem. If this is not satisfied for some connected component, at least some of the considered eigenproblems will have eigenfunctions that are piecewise constant per component and correspond to eigenvalue 0. These should not be included in a basis for $W(\mathcal{M})$, as they have vanishing $W-$norm. We avoid this issue by rejecting eigenfunctions with eigenvalues below a certain small threshold. Note that for components with \textit{one} landmark we only impose the Steklov condition on the corresponding circle, omitting the second line of Eq. \eqref{Eq:DirichletSteklov}.

\begin{figure}
    \includegraphics[width=\linewidth]{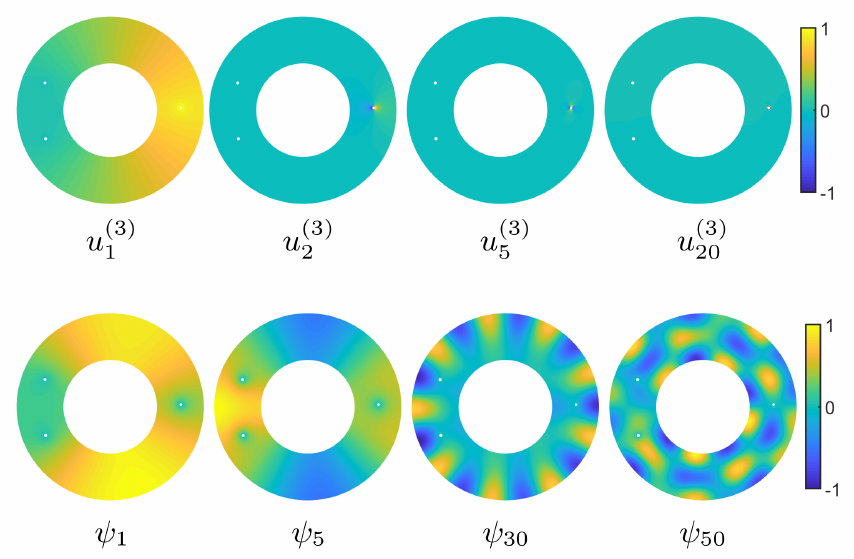}
    \caption{Dirichlet-Steklov (top row) and Dirichlet Laplacian (bottom row) eigenfunctions on an annulus with three landmark circles and Neumann conditions on inner and outer boundaries. The Dirichlet-Steklov eigenfunctions correspond to the landmark on the right. Notice that as the eigenvalues increase, the Dirichlet-Steklov eigenfunctions quickly concentrate around the corresponding landmark, unlike the Dirichlet Laplacian eigenfunctions that remain distributed in the bulk of the annulus.\vspace{-2mm}}
    \label{Fig:AnnulusBasisComparison}
\end{figure}

\subsection{Structure of the Functional Map} \label{Sec:FmapStructure}

Recall that our ultimate goal is to compute a near-conformal diffeomorphism $\varphi: \mathcal{N} \to \mathcal{M}$ that preserves the landmarks. Recall also that we propose to use functional map methods to find it. In this section we  translate the structural properties of $\varphi$ into properties of its pullback $F_{_{\mathcal{MN}}}$ which helps us to restrict the space of admissible functional maps, which is crucial for our approach.

We begin on a technical note. Since we have replaced landmark points with landmark circles, the notion of landmark preservation has to be slightly adjusted. We no longer can claim something as simple as $\varphi(\gamma_i^{_{\mathcal{N}}}) = \gamma_i^{_{\mathcal{M}}}$ for all $i$, as the landmark points are no longer part of the considered shapes. Instead we impose that $\varphi$ restricts to a diffeomorphism on corresponding landmark circles. That is, $\varphi: \mathcal{N} \to \mathcal{M}$ is a diffeomorphism and for each $i$, $\varphi |_{\Gamma_i^{_{\mathcal{N}}}}: \Gamma_i^{_{\mathcal{N}}} \to \Gamma_i^{_{\mathcal{M}}}$ is also a diffeomorphism.

Now, suppose that $\varphi$ is indeed a conformal map. Then, $F_{_{\mathcal{MN}}}$ satisfies the following lemma.

\begin{lemma}[Structure of $F_{_{\mathcal{MN}}}$] \label{Lemma:StructureF}
Let $F_{_{\mathcal{MN}}}: W(\mathcal{M}) \to W(\mathcal{N})$ be the pullback of a conformal diffeomorphism that preserves the landmark circles in the sense described above. Then, $F_{_{\mathcal{MN}}}$ maps
\begin{enumerate}
	\item $\mathcal{G(M)}$ to $\mathcal{G(N)}$,
	\item $\mathcal{H}_j(\mathcal{M})$ to $\mathcal{H}_j(\mathcal{N})$ for all $j$.
\end{enumerate}
\end{lemma}
\begin{proof}
See App.~\ref{Sec:ProofOfStructureLemma}.
\end{proof}
The above lemma provides necessary, but not sufficient conditions for $F_{_{\mathcal{MN}}}$ to be the pullback of a diffeomorphism preserving the landmark circles. Nonetheless, we will use properties $(1)$ and $(2)$ of Lemma \ref{Lemma:StructureF} to structure our search for $F_{_{\mathcal{MN}}}$.

 From now on, we only consider functional maps that satisfy statements $(1)$ and $(2)$ of Lemma \ref{Lemma:StructureF}. This can be seen as $k+1$ separate maps, one for each landmark subspace $\mathcal{H}_i$ and one for the orthogonal complement $\mathcal{G}$, assembled into one block-diagonal functional map. Intuitively, this keeps the overall map tethered to the landmarks.

\paragraph*{Landmark preservation}
At this point, it is worth explaining what we mean when we say that our method preserves the landmark correspondences \emph{in the discrete setting}. Indeed, the challenge of landmark preservation is to not merely enforce the condition $\varphi(\gamma_i^{_{\mathcal{N}}}) = \gamma_i^{_{\mathcal{M}}}$, but to also obtain a smooth (or at least continuous) map in the neighborhood of the landmarks (notice that we required $\varphi$ to be a diffeomorphism when discussing the smooth setting). Our method achieves this by using a functional basis whose elements are well suited to describe smooth functions near the landmarks (recall the decay of the Dirichlet-Steklov eigenfunctions away from the Steklov boundary depicted in Fig.
\ref{Fig:AnnulusBasisComparison}). By enforcing the functional map structure of Lemma
\ref{Lemma:StructureF}  during the entire solution process, we promote vertex-to-vertex maps that smoothly map the neighborhoods of the landmarks of $\mathcal{N}$ to the corresponding neighborhoods on $\mathcal{M}$, the smoothness of the map reflecting the smoothness of the functional basis. Furthermore, we reinsert the original pointwise landmarks at the end of the solution process to preserve the initial landmarks exactly. Recall that the landmark vertices are excluded from the meshes the moment the landmark circles are introduced.

%% file: Sections/s05_Energy.tex
\section{Functional Map Energy} \label{sec:EnergyFormulation}

The previous section describes our landmark adapted basis construction, and the block-diagonal structure of landmark-preserving conformal maps when expressed in this basis. In this section we specify the optimization problem that we will solve in order to obtain landmark-preserving maps between triangle meshes.

Recall that we propose to look for conformal maps, which can be characterized in terms of the Dirichlet form ($W(\mathcal{M})$ inner product).

\begin{theorem} \label{Th:ConformalCharacterization}
	Let $\varphi: \mathcal{N} \to \mathcal{M}$ be a diffeomorphism between oriented Riemannian surfaces with pullback $F_{_{\mathcal{MN}}}: W(\mathcal{M}) \to W(\mathcal{N})$. Then, $\varphi$ is conformal if and only if
	\begin{equation}
	\langle u,v \rangle_{_{W(\mathcal{M})}} = \langle F_{_{\mathcal{MN}}}~u, F_{_{\mathcal{MN}}}~v \rangle_{_{W(\mathcal{N})}}~~,~~\forall u,v \in W(\mathcal{M})~.
	\end{equation}
\end{theorem}
\begin{proof}
	See \cite{rustamov2013map}.
\end{proof}

In practice we do not expect to obtain an exact equality of the inner products as described in the previous theorem. Instead, we will search for $\varphi$ and $F_{_{\mathcal{MN}}}$ by relaxing the above equality to a minimization problem. Let $\Phi^{_{\mathcal{M}}}$ and $\Phi^{_{\mathcal{N}}}$ denote reduced (finite dimensional) functional bases for $W\mathcal{(M)}$ and $W\mathcal{(N)}$, respectively. These bases consist of the eigenfunctions of the Dirichlet Laplacian and Dirichlet-Steklov eigenproblems corresponding to small eigenvalues. The precise size of the bases is discussed in App.~\ref{Sec:appendix_num_LB_DS_eigs}.

From now on, we concentrate our attention on the discrete case. Namely, $\mathcal{M}$ and $\mathcal{N}$ will now denote oriented manifold triangle meshes. Letting $\left\langle \Phi^{_{\mathcal{M}}}, \Phi^{_{\mathcal{M}}} \right\rangle_{_{W(\mathcal{M})}}$ be the matrix of all inner products of the normalized basis vectors of $\Phi^{_{\mathcal{M}}}$, we relax the equality of Theorem \ref{Th:ConformalCharacterization}  to the minimization of the following energy term:
\begin{equation} \label{Eq:ConformalTerm}
\begin{aligned}
    E_{c}(F_{_{\mathcal{MN}}}) = \Big\| &\left\langle \Phi^{_{\mathcal{M}}}, \Phi^{_{\mathcal{M}}} \right\rangle_{_{W(\mathcal{M})}} \\ &- \left\langle F_{_{\mathcal{MN}}}~\Phi^{_{\mathcal{M}}}, F_{_{\mathcal{MN}}}~\Phi^{_{\mathcal{M}}} \right\rangle_{_{W(\mathcal{N})}} \Big\|^2_F~.
\end{aligned}
\end{equation}

\noindent We call this the \emph{conformal term} of the energy. Here, as well as everywhere else in this text, $\|\cdot\|_F$ denotes the Frobenius norm.

Having covered the conformality of the map, it remains to rephrase the restriction of $F_{_{\mathcal{MN}}}$ to pullbacks of landmark-preserving diffeomorphisms. This assumption cannot be exactly imposed in the discrete case.
Still, we would like $F_{_{\mathcal{MN}}}$ to exhibit the properties of such a map. In order to do so, we complete our energy by specifying two \emph{structural terms}. Specifically, the first term promotes $F_{_{\mathcal{MN}}}$ being a \emph{proper} functional map (i.e., the pullback of a vertex-to-vertex map), as recently defined in \cite{DiscreteSolver}, and the second promotes the invertibility of $F_{_{\mathcal{MN}}}$ \cite{eynard2016coupled}.

Let $\Pi_{_{\mathcal{NM}}}$ denote the vertex-to-vertex map from $\mathcal{N}$ to $\mathcal{M}$ expressed as a matrix (i.e. a binary matrix that contains exactly one $1$ per row). Then, $F_{_{\mathcal{MN}}}$ should satisfy:
\begin{equation} \label{Eq:FmapFromP2P}
F_{_{\mathcal{MN}}} = \left(\Phi^{_{\mathcal{N}}}\right)^{+} \Pi_{_{\mathcal{NM}}} \Phi^{_{\mathcal{M}}}~,
\end{equation}

\noindent where $\left(\Phi^{_{\mathcal{N}}}\right)^{+}$ denotes the pseudoinverse of $\Phi^{_{\mathcal{N}}}$, or in other words, the $W \mathcal{(N)}$ projection onto the reduced basis $\Phi^{_{\mathcal{N}}}$. As before, we relax the equality into an energy to be optimized:
\begin{equation} \label{eq:PropernessTerm}
E_{p}(F_{_{\mathcal{MN}}}, \Pi_{_{\mathcal{NM}}}) = \left\| \left(\Phi^{_{\mathcal{N}}}\right)^{+} \Pi_{_{\mathcal{NM}}} \Phi^{_{\mathcal{M}}} - F_{_{\mathcal{MN}}} \right\|^2_F~.
\end{equation}

\noindent We call this the \emph{properness term} of the energy. Notice that we have expressed the energy as a function of both $F$ and $\Pi$. We do so as we will have to consider these two objects as independent variables when minimizing the energy. The exact meaning of this is discussed in Sec.~\ref{Sec:SolvingTheProblem}.

In addition to $F_{_{\mathcal{MN}}}$ arising from a point-to-point map, we would also like for it to be invertible. For this, we consider two maps $F_{_{\mathcal{MN}}}: W(\mathcal{M}) \to W(\mathcal{N})$ and $F_{_{\mathcal{NM}}}: W(\mathcal{N}) \to W(\mathcal{M})$, the latter arising from a vertex-to-vertex map $\Pi_{_{\mathcal{MN}}}: \mathcal{M} \to \mathcal{N}$. Thus, in what follows, we will be simultaneously optimizing for maps going in both directions between the shapes. With $I$ being the identity matrix, the invertibility condition is, of course:
\begin{equation}
	\begin{aligned}
	F_{_{\mathcal{NM}}}F_{_{\mathcal{MN}}} = I~,\\
	F_{_{\mathcal{MN}}}F_{_{\mathcal{NM}}} = I~.
	\end{aligned}
\end{equation}

\noindent Once again, we convert the above into minimization form. The \emph{invertibility term} corresponding to the first line above is

\begin{equation} \label{eq:InvertibilityTerm}
\begin{aligned}
E_{_{I,\mathcal{MN}}} (F_{_{\mathcal{MN}}}, F_{_{\mathcal{NM}}}) = &\left\| F_{_{\mathcal{MN}}}F_{_{\mathcal{NM}}}  - I \right\|^2_F~.
\end{aligned}
\end{equation}

\noindent The invertibility term $E_{_{I,\mathcal{NM}}}$ is defined analogously.

In sum, our search for the correspondence between $\mathcal{M}$ and $\mathcal{N}$ will involve the joint minimization of the energy

\begin{equation}\label{Eq:MinimizationProblem}
\begin{aligned}
E_{_\mathcal{MN}} = ~& a_{C}~ E_{c}(F_{_{\mathcal{MN}}}) ~+\\
 &a_{P}~ E_{p}(F_{_{\mathcal{MN}}}, \Pi_{_{\mathcal{NM}}}) ~+\\
 &a_{I}~ E_{_{I,\mathcal{MN}}} (F_{_{\mathcal{MN}}},F_{_{\mathcal{NM}}})
\end{aligned}
\end{equation}

\noindent and an analogously defined energy $E_{_\mathcal{NM}}$. Here, $a_C,a_P$ and $a_I$ are nonnegative tunable weights controlling the relative strength of the conformality, properness and invertibility terms, respectively. Different values of these parameters are explored in App.~\ref{Sec:appendix_energy_weights}.

The above energy is conformally invariant in the following sense.

\begin{lemma}[Energy Invariance]\label{lemma:EnergyInvariance}
The conformality, properness and invertibility terms of the energy (Eqs. \eqref{Eq:ConformalTerm}, \eqref{eq:PropernessTerm} and \eqref{eq:InvertibilityTerm}), as well as the energy (their weighted sum, Eq. \eqref{Eq:MinimizationProblem}) are invariant under (combinations of) the following transformations:
\begin{enumerate}
    \item Conformal transformations of the meshes keeping the reduced bases fixed.
    \item Orthogonal transformations of the reduced bases.
\end{enumerate}
\end{lemma}
\begin{proof}
Let the functional and vertex-to-vertex maps be fixed. Since conformal transformations leave the $W$ inner product invariant, the energy terms are conformally invariant for a fixed choice of functional basis. Statement $1.$ is now proven. Statement $2.$ follows from the fact the Frobenius norm is invariant under orthogonal transformations.
\end{proof}

Note that, in the lemma above, conformal transformations and changes of basis are treated as independent. In practice, they are not, as reduced bases are usually mesh-dependent. 
Thus, the change of basis induced by a conformal transformation may fail to be orthogonal and then Lemma \ref{lemma:EnergyInvariance} will not apply. As long as one works with reduced rather than full bases (\emph{i.e.} spanning all functions of the mesh), the invariance of the energy under conformal transformations is therefore only approximately guaranteed.

The conformal invariance of the energy can be violated in another way. Suppose that the recipe for constructing the reduced bases produces non-orthogonal bases. Then, the change of basis induced by a conformal transformation may fail to be orthogonal even when full bases are used.

Of course, we raise the previous two issues precisely because our method uses non-orthogonal reduced bases. Thus, Lemma \ref{lemma:EnergyInvariance} does not offer a full guarantee of conformal invariance for our energy. Still, the bases that we use turn out to be very nearly orthogonal and thus the energy that we employ remains approximately conformally invariant. Obtaining a truly conformally invariant energy (at least up to basis truncation) is a subject for future work.

%% file: Sections/s06_SolvingProblem.tex
\section{Solving the Problem} \label{Sec:SolvingTheProblem}

In this section, we propose an efficient approach for the optimization problem posed in Eq. \eqref{Eq:MinimizationProblem}. Our approach is inspired by a discrete optimization strategy, first suggested in  \cite{melzi2019zoomout} and recently extended to other general energies \cite{DiscreteSolver}. The general idea is to recast the problem in a way that makes every iteration of the optimization into a nearest-neighbor search. The overall process then consists of two qualitatively different parts. First, an initial guess of the correspondence is obtained. Then, the correspondence is refined via the iterative process mentioned above. These steps are explained in Secs.~\ref{Sec:InitialGuess} and  \ref{Sec:ConversionToNearestNeighbor}, respectively.

\subsection{Initial Correspondence} \label{Sec:InitialGuess}

In this section we explain how we obtain an initial guess of the functional maps $F_{_{\mathcal{MN}}}$ and $F_{_{\mathcal{NM}}}$. A common way to initialize functional maps with landmarks is via descriptor preservation \cite{ovsjanikov2012functional,ren2018continuous}. However, common descriptors such as HKS or WKS \cite{sun2009concise,aubry2011wave} strongly rely on the isometry assumption and moreover the initial functional map is not guaranteed to respect landmark correspondences exactly. To overcome this, we propose a simple and lightweight initialization scheme.

Recall that our approach upgrades landmark correspondence to landmark circle correspondence. Moreover, in the smooth setting, we require this correspondence to be a diffeomorphism. We now make the assumption that the correspondence between landmark circles can be seen as a rotation of one circle to match the other.

Specifically, we label the vertices of each landmark circle in counter-clockwise order using the outward-facing normal orientation. We can then assign each vertex in a landmark circle coordinates in $[0,1)$. All that remains is to ensure that the origin of this coordinate system is placed consistently on both shapes. In other words, the matching of two corresponding landmark circles reduces to finding an appropriate shift of one of the parametrizations.

We propose to align the parametrizations of the boundary circles such that the landmark circles are consistently oriented relative to the other landmarks. In order to do so, we construct functions on the landmark circles that have maxima in directions roughly pointing towards other landmarks. Consider the following problem:
\begin{equation} \label{Eq:LandmarklHarmonics}
    \begin{aligned}
        &\Delta h_i = 0~~,~~i=1 ... k~,\\
        &h_i\big|_{\Gamma_j} = \delta_{ij}~.
    \end{aligned}
\end{equation}
This results in $k$ harmonic functions, one for each landmark, where each function $h_i$ equals 1 on the boundary of landmark circle $i$, and zeros on the boundaries of other circles. As they stand, these functions are constant on each landmark circle. Their normal derivatives, however, are not. In essence, we use $\partial_n\big|_{_{\Gamma_i}} h_j $ as an indication of the direction one should take from $\Gamma_i$ to reach $\Gamma_j$. See Fig. \ref{Fig:HarmonicOnDisk} for an illustration. This is similar in spirit to the Geodesics in Heat construction \cite{crane2013geodesics}, where gradients of solutions to the heat transfer problem are used to construct approximate geodesics.

\begin{figure}
    \centering \includegraphics[width=0.45\linewidth]{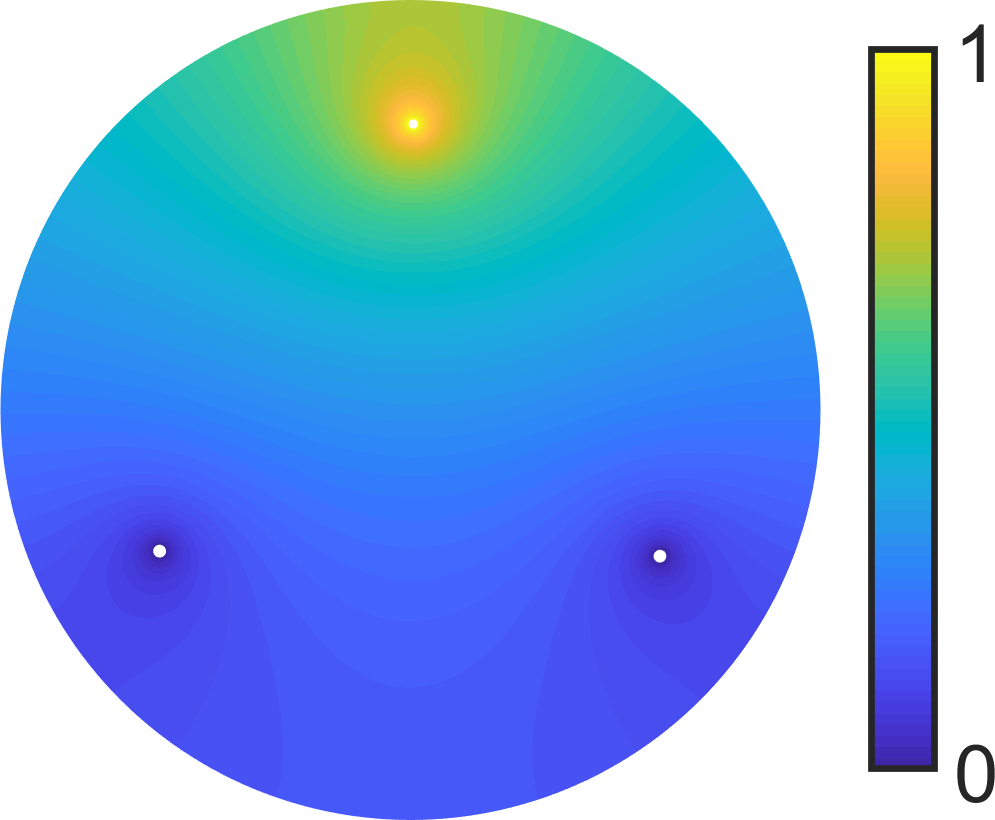}
    \caption{Harmonic function satisfying Eq. \eqref{Eq:LandmarklHarmonics} corresponding to the top landmark on a disk with three landmark circles. Notice that the gradient of the function roughly points towards the top landmark. Hence, its normal derivatives at the bottom landmark circles can be seen as specifying the direction towards the top landmark.}
    \label{Fig:HarmonicOnDisk}
\end{figure}

Denoting the landmark circle coordinates on $\Gamma_{i}^{_{\mathcal{M}}}$ by $\theta_i$, we select the optimal shift $\alpha_i$ by solving:
\begin{equation}
    \begin{aligned}
            \alpha_i = \argmin_\alpha \Big( \sum_{j \neq i} \Big\| &\partial_n\big|_{_{\Gamma_i^{_{\mathcal{M}}}}} h_j^{_{\mathcal{M}}}(\theta_i ) \\
            - &\partial_n\big|_{_{\Gamma_i^{_{\mathcal{N}}}}} h_{j}^{_{\mathcal{N}}}( \text{mod}(\theta_i - \alpha, 1) )    \Big\|^2_{L_{2}(\Gamma_{i}^{_{\mathcal{M}}})} \Big).
    \end{aligned}
\end{equation}
In this problem, we consider each landmark $i$ and examine directions to all other landmarks (via normal derivatives). We then align the coordinates of landmarks on $\mathcal{M}$ and $\mathcal{N}$ so that these directions align in the best possible way. This problem can be solved simply by directly examining all possible shifts and taking the optimum. In order to gain robustness to changes in triangulation, we first project the normal derivatives (as circular functions $\partial_n \big|_{\Gamma_i^{_{\mathcal{M}}}} h_j^{_{\mathcal{M}}}(\theta)$)
    onto the reduced basis in order to remove spurious high frequency components. Recall that this makes sense as the Dirichlet-Steklov eigenfunctions belonging to landmark $\Gamma_i$ form a basis for $L_2(\Gamma_i)$.

Converting the optimal shifts $\alpha_i$ into vertex-to-vertex correspondences on the landmark circles is a matter of a nearest-neighbor search between the circular coordinates of the vertices of $\Gamma_{i}^{_{\mathcal{M}}}$ and the $\alpha_i$-shifted coordinates of the vertices of $\Gamma_{i}^{_{\mathcal{N}}}$.

It remains to convert the resulting vertex-to-vertex map into a functional map. Once again, recall that our reduced basis contains an $L_2$ basis for each landmark circle. Thus, by using an expression of the form of Eq. \eqref{Eq:FmapFromP2P} we can construct functional maps between $\mathcal{H}_i(\mathcal{M})$ and $\mathcal{H}_i(\mathcal{N})$. Assembling the resulting maps into block-diagonal matrices gives our initial guesses of $F_{_{\mathcal{MN}}}$ and $F_{_{\mathcal{NM}}}$.

In App.~\ref{Sec:appendix_alternative_initialization_comparison}, we compare this approach to two alternative initialization strategies, and demonstrate its relative advantages.

\subsection{Energy Minimization via Nearest Neighbor Search} \label{Sec:ConversionToNearestNeighbor}
The procedure described in Sec.~\ref{Sec:InitialGuess} provides a descriptor-free initial guess for the functional map. In this section we describe a refinement method that significantly improves the map.

Recall that we are looking for a vertex-to-vertex map by minimizing an energy that depends on both the point-to-point and the associated functional map (pullback). In \cite{melzi2019zoomout} it is observed that a particular case of such problems can be efficiently solved by considering the two maps as being independent variables. This observation was recently extended to a wide range of energies in \cite{DiscreteSolver}. Following this line of work we will move all of the difficult optimization on the side of the vertex-to-vertex map and use Eq. \eqref{Eq:FmapFromP2P} to restore the relationship between the maps.

Our main tool is the following result, standard in functional maps literature \cite{ezuz2017deblurring,DiscreteSolver}, which allows one to reduce optimization problems of a certain form to nearest neighbor searches.

\begin{lemma}\label{Lemma:ConvertToNN}
Let $A$ be a symmetric positive-definite matrix inducing the matrix norm $\|M\|^2_A = \text{Tr}(M^T A M)$. Let $\Phi$ be a reduced basis orthogonal with respect to $A$, that is $\Phi^T A \Phi = I$. Then, given $n$ pairs of matrices $X_i$ and $Y_i$, the following two expressions are equal:
\begin{enumerate}
	\item $ \sum_i^n \left( \| \Phi^{T}A \Pi X_i - Y_i \|^2_F + \| (I - \Phi \Phi^T A) \Pi X_i \|^2_F \right)$
	\item $ \sum_i^n \| \Pi X_i - \Phi Y_i \|^2_{A}$
\end{enumerate}
\noindent Moreover, if $A$ is diagonal, minimizing the above expressions over matrices $\Pi$ that reflect point-to-point maps (i.e., binary matrices that contain exactly one 1 per row) is equivalent to
	\begin{equation}
		\min_{\Pi}  \sum_i^n \| \Pi X_i - \Phi Y_i \|^2_{F}~.
	\end{equation}
\noindent This problem can be solved via nearest-neighbor search between the rows of the concatenated matrices $\left[ X_1~...~X_n\right]$ and $\left[ \Phi Y_1 ~...~\Phi Y_n \right]$.
\end{lemma}
\vspace{-2mm}
\begin{proof}
	See \cite{ezuz2017deblurring} for a proof of a special case and \cite{DiscreteSolver} (Lemma 4.1) for the general statement.
\end{proof}

We first convert the energy of Eq. \eqref{Eq:MinimizationProblem} into the form used in the above lemma. For brevity's sake, we will only develop the expression for $F_{_{\mathcal{MN}}}$ and $\Pi_{_{\mathcal{NM}}}$. The expression for the pair  $F_{_{\mathcal{NM}}}$ and $\Pi_{_{\mathcal{MN}}}$ is analogous. As mentioned in Sec.~\ref{Sec:LandmarkAdaptedBasis}, we approximate the functional bases $\Phi^{_{\mathcal{M}}}$ and $\Phi^{_{\mathcal{N}}}$ to be orthonormal with respect to the Dirichlet form. Then, the energy minimized by the desired $F_{_{\mathcal{MN}}}$ and $\Pi_{_{\mathcal{NM}}}$ becomes:
\begin{equation}
\begin{aligned}
E_{_{\mathcal{MN}}} = ~~~~~~~ & a_{c} \left\| I -  \left( F_{_{\mathcal{MN}}} \right)^T F_{_{\mathcal{MN}}} \right\|^2_F\\
+ &a_{p}\left\| \left(\Phi^{_{\mathcal{N}}}\right)^{T} W^{_{\mathcal{N}}} \Pi_{_{\mathcal{NM}}} \Phi^{_{\mathcal{M}}} - F_{_{\mathcal{MN}}} \right\|^2_F\\
+ &a_{I} \left\| F_{_{\mathcal{MN}}}F_{_{\mathcal{NM}}}  - I \right\|^2_F ~.
\end{aligned}
\end{equation}
 Here, we used the approximation of basis orthonormality in two ways. First, we used it to evaluate the inner products in the conformality term (first line of the above equation). Second, we used it to express $\left(\Phi^{_\mathcal{N}} \right)^{+} = \left(\Phi^{_{\mathcal{N}}}\right)^{T} W^{_{\mathcal{N}}}$, where $W^{_{\mathcal{N}}}$ is the so-called cotangent Laplacian on $\mathcal{N}$, which also corresponds to the piecewise linear finite element discretization of the Dirichlet form. We are still a few modifications away from being able to apply Lemma \ref{Lemma:ConvertToNN} to this problem.

We obtain the desired form for the expression by replacing certain instances of $F_{_{\mathcal{MN}}}$ with its expression in terms of the vertex-to-vertex map $\Pi_{_{\mathcal{NM}}}$:  $\left(\Phi^{_{\mathcal{N}}}\right)^{T} W^{_{\mathcal{N}}} \Pi_{_{\mathcal{NM}}} \Phi^{_{\mathcal{M}}}$. By using the fact that  $\|I - F^T F\|^2_F = \|FF^T - I\|^2_F$ and making this replacement, we obtain:
\begin{equation}
\begin{aligned}
E_{_{\mathcal{MN}}} = ~~~~~~~ & a_{c} \left\|\left(\Phi^{_{\mathcal{N}}}\right)^{T} W^{_{\mathcal{N}}} \Pi_{_{\mathcal{NM}}} \Phi^{_{\mathcal{M}}} \left( F_{_{\mathcal{MN}}} \right)^T - I \right\|^2_F\\
+ &a_{p}\left\| \left(\Phi^{_{\mathcal{N}}}\right)^{T} W^{_{\mathcal{N}}} \Pi_{_{\mathcal{NM}}} \Phi^{_{\mathcal{M}}} - F_{_{\mathcal{MN}}} \right\|^2_F\\
+ &a_{I} \left\| \left(\Phi^{_{\mathcal{N}}}\right)^{T} W^{_{\mathcal{N}}} \Pi_{_{\mathcal{NM}}} \Phi^{_{\mathcal{M}}} F_{_{\mathcal{NM}}}  - I \right\|^2_F ~.
\end{aligned}
\end{equation}

\noindent Now, all of the terms of the above are of the form $\| \Phi^{T}A \Pi X_i - Y_i \|^2_F$, with $W^{_{\mathcal{N}}}$ playing the role of the matrix $A$. Our energy is thus of the form of line $(1)$ of Lemma \ref{Lemma:ConvertToNN}, up to three terms of the form $\| (I - \Phi \Phi^T A) \Pi X_i \|^2$. Notice that $(I - \Phi \Phi^T A)$ is the orthogonal projection onto the orthogonal complement of the reduced (approximately) orthonormal basis $\Phi$. Thus, this term can be seen as a regularizer penalizing solutions lying outside of the considered reduced basis. Indeed, this is how this term is was originally introduced in \cite{ezuz2017deblurring}. Consequently, by implicitly introducing the appropriate regularizers we can use the first part of Lemma \ref{Lemma:ConvertToNN} to obtain the following expression for the energy:
\begin{equation}
\begin{aligned}
E_{_{\mathcal{MN}}} = ~~~~~~~  a_{c} &\left\| \Pi_{_{\mathcal{NM}}} \Phi^{_{\mathcal{M}}} \left( F_{_{\mathcal{MN}}} \right)^T - \Phi^{_{\mathcal{N}}} \right\|^2_{W^{^{_{\mathcal{N}}}}}\\
+ a_{p} &\left\| \Pi_{_{\mathcal{NM}}} \Phi^{_{\mathcal{M}}} - \Phi^{_{\mathcal{N}}}  F_{_{\mathcal{MN}}} \right\|^2_{W^{^{_{\mathcal{N}}}}}\\
+ a_{I} &\left\|  \Pi_{_{\mathcal{NM}}} \Phi^{_{\mathcal{M}}} F_{_{\mathcal{NM}}}  - \Phi^{_{\mathcal{N}}} \right\|^2_{W^{^{_{\mathcal{N}}}}} ~.
\end{aligned}
\end{equation}

\noindent At this point we are forced to make an approximation. Namely, we assume that the second part of the lemma applies, which would normally require $W^{_{\mathcal{N}}}$ to be diagonal. In other words, we convert the problem into a nearest neighbor search without having the guarantee of the equivalence of solutions. Despite this approximation, we have observed that the resulting approach works remarkably well in practice.

This finally brings us to the procedure that we use to minimize the energy. As mentioned above, we will consider the functional and vertex-to-vertex maps as independent variables. Thus, given functional maps $F_{_{\mathcal{MN}}}$ and $F_{_{\mathcal{NM}}}$, the point-to-point map $\Pi_{_{\mathcal{NM}}}$ can be found by solving the nearest-neighbor search problem:
\begin{equation} \label{Eq:GetPiNM}
\begin{aligned}
    \Pi_{_{\mathcal{NM}}} = \text{NNS}\Big( \big[ &\Phi^{_{\mathcal{M}}} \left( F_{_{\mathcal{MN}}} \right)^T &&  \Phi^{_{\mathcal{M}}} &&&\Phi^{_{\mathcal{M}}} F_{_{\mathcal{NM}}}\big],\\
    \big[ &\Phi^{_{\mathcal{N}}} && \Phi^{_{\mathcal{N}}}  F_{_{\mathcal{MN}}} &&&\Phi^{_{\mathcal{N}}} \big] \Big)~.
\end{aligned}
\end{equation}
Here NNS($A,B$) denotes a set of nearest neighbor problems: for each row of $B$ among the rows of $A$. The vertex-to-vertex map $\Pi_{_{\mathcal{MN}}}$ can be obtained analogously.
In sum, minimizing the energy with respect to the vertex-to-vertex maps is also a recipe for converting functional maps into vertex-to-vertex maps, while taking into account the original functional map energy.

We are now ready to formulate the optimization algorithm. Following \cite{melzi2019zoomout}, the overall procedure is based on an iterative spectral upsampling of the functional map. Specifically, we iteratively convert the functional map into a vertex-to-vertex map while increasing the size of the reduced basis. As explained earlier (see Fig. \ref{Fig:AnnulusBasisComparison}), the Dirichlet-Steklov functions are concentrated near the landmark circles. Thus, increasing their number does not provide much additional information about the map in the bulk of the shapes. Thus, we only increase the number of Dirichlet Laplacian eigenfunctions.

Beginning from the initial functional maps $F_{_{\mathcal{MN}}}$ and $F_{_{\mathcal{NM}}}$ obtained in Sec.~\ref{Sec:InitialGuess}, we proceed as follows.

\begin{enumerate}
    \item Convert $F_{_{\mathcal{MN}}}$ and $F_{_{\mathcal{NM}}}$ into $\Pi_{_{\mathcal{NM}}}$ and $\Pi_{_{\mathcal{MN}}}$ via Eq. \eqref{Eq:GetPiNM}.
    \item Increase the reduced bases $\Phi^{_{\mathcal{M}}}$ and $\Phi^{_{\mathcal{N}}}$ by including $k_{step}$ additional Dirichlet Laplacian eigenfunctions.
    \item Update the functional maps to the new basis size via $F_{_{\mathcal{MN}}} = \left( \Phi^{_{\mathcal{N}}} \right)^+ \Pi_{_{\mathcal{NM}}} \Phi^{_{\mathcal{M}}}$ and $F_{_{\mathcal{NM}}} = \left( \Phi^{_{\mathcal{M}}} \right)^+ \Pi_{_{\mathcal{MN}}} \Phi^{_{\mathcal{N}}}$.
    \item Iterate steps $(1)$ to $(3)$ until the desired basis size is reached.
    \item Repeat step $(1)$ using only the original non-landmark vertices. This produces a vertex-to-vertex map between the original meshes, landmarks excluded.
    \item Insert the landmark correspondence into the vertex-to-vertex map.
\end{enumerate}

\emph{A Fast Approximation.} We conclude this section by proposing an acceleration strategy to perform the nearest-neighbor search. The method proposed here is unprincipled, but is validated by both the overall quality of our results and explicit tests found in App.~\ref{Sec:appendix_principled_vs_fast}. The method proposed below is the only one used in the main text of this paper.

In the language of Lemma \ref{Lemma:ConvertToNN}, we propose to replace the nearest neighbor search between the concatenated matrices $\left[ X_1~...~X_n\right]$ and $\left[ \Phi Y_1 ~...~\Phi Y_n \right]$ by a nearest neighbor search between the summed matrices $ X_1 + ... + X_n$ and $\ \Phi Y_1 + ... +\Phi Y_n$. This corresponds to solving the following problem:
\begin{equation}
	\min_{\Pi}  \left\| \sum_i^n  \Pi X_i - \Phi Y_i \right\|^2_{F}~.
\end{equation}
This reformulation helps to decrease the dimensionality of the nearest neighbor searches. Essentially, we assume that the different energy terms will not cancel each other. The payoff for this approximation is that the matrices involved in the nearest-neighbor search become $n$ times smaller. In our case, there are $n=3$ energy terms. The experiments in App.~\ref{Sec:appendix_principled_vs_fast} show that this reduction in matrix size results in a slightly more than threefold speed-up.

%% file: Sections/s07_Evaluation.tex
\section{Evaluation}\label{Sec:evaluation}
We evaluate our method\footnote{Our code is available at \url{https://github.com/mpanine/DirichletSteklovLandmarkMatching}} on standard shape matching datasets, which we describe in Sec.~\ref{sec:evaluation_datasets}. We first analyze the parameters involved in our computations (Sec.~\ref{Sec:parameter_study}). Second, we conduct an in-depth evaluation to compare our method to state-of-the-art approaches on shape matching benchmarks (Sec.~\ref{Sec:benchmarks}). 

For our quantitative evaluation in Fig.~\ref{fig:remeshing_stability} (right), Fig.~\ref{fig:FAUST_TOSCA_evaluation}, Fig.~\ref{fig:TOSCA_NONISO_SHREC20_evaluation} and Fig.~\ref{fig:SHREC19_evaluation}, we follow the commonly-used protocol, introduced in \cite{kim2011blended} by plotting the percentage of correspondences below a certain geodesic distance threshold from the ground truth.

\subsection{Datasets}\label{sec:evaluation_datasets}
We perform all our experiments on the following datasets.
\par
\textbf{FAUST~\cite{bogo2014faust}.} This dataset contains models of ten different humans in ten poses each. Despite the variability in the body types of said humans, this dataset is typically considered as near-isometric. We remesh the shapes of the dataset to shapes with approximately $5$K vertices and use $300$ shape pairs following the procedure of the authors of~\cite{ren2018continuous}. Note that the shapes in question are remeshed independently and do not share the same connectivity. 
\par
\textbf{TOSCA~\cite{bronstein2008numerical}.} This dataset consists of meshes of humans and animals. Following~\cite{ren2018continuous}, we split this dataset into  isometric and non-isometric shape pairs. We call the resulting datasets TOSCA isometric ($284$ shape pairs) and TOSCA non-isometric ($95$ shape pairs) respectively. The shapes of these datasets are remeshed independently to count around $5$K vertices per shape. Once again, the remeshed shapes have distinct connectivity.
\par
\textbf{SHREC'19~\cite{melzi2019shrec}.} This challenging dataset is composed of human shapes with high variability in pose, vertex count (ranging from 5K to 200K vertices) and topology (some shapes are watertight manifold meshes whereas other have holes and other surface noise sources).
\par
\textbf{FAUST ``Wild''~\cite{sharp2020diffusion}.} This dataset is a variant of FAUST in which challenging differences in connectivity are introduced via remeshing. We use the following types of remeshing of the dataset: a uniform isotropic remeshing (\textit{iso}), a remeshing where randomly sampled regions are refined (\textit{dense}), and the remeshing proposed in~\cite{QSlim} (\textit{qes}). Finally, we consider correspondences \textit{across} the $20$ template models of the dataset instead of solely considering the initial template shape as the source shape.
\par
\textbf{SHREC'20~\cite{dyke2020shrec}.} This dataset proposes a collection of $14$ animal shapes with a set of landmarks determined by experts. The animal pairs contain parts in correspondence with highly non-isometric deformations. We only consider the correspondences between full shapes for our experiments (test sets $1$ to $4$).

\subsection{Parameter Study}\label{Sec:parameter_study}
We present here the main results concerning the parameters of our method. Other minor experiments on this topic are presented in App.~\ref{Sec:appendix_additional_parameter_study} (influence of the weights in the energy, qualitative illustration of the impact of landmark placement, near-orthogonality assessment for our basis and study of the effect of basis size).
\subsubsection{Radius $r_f$}\label{Sec:r_f_study}
The construction of the landmark boundaries $\Gamma_i$ explained in App.~\ref{sec:DiscreteLandmarkCircle} relies on the user-defined scalar parameter $r_f\in(0,1)$. In Fig.~\ref{fig:param_study_r_f}, we study the influence of $r_f$ on the geodesic matching error averaged on the TOSCA non-isometric dataset, with $7$ landmark correspondences at their standard locations (see App.~\ref{Sec:appendix_evaluation_setup_details}). It demonstrates empirically that this parameter has no significant impact on the matching performance. We therefore set $r_f=0.5$ in all our other experiments.

\begin{figure}
    \centering
    \includegraphics[height=3cm]{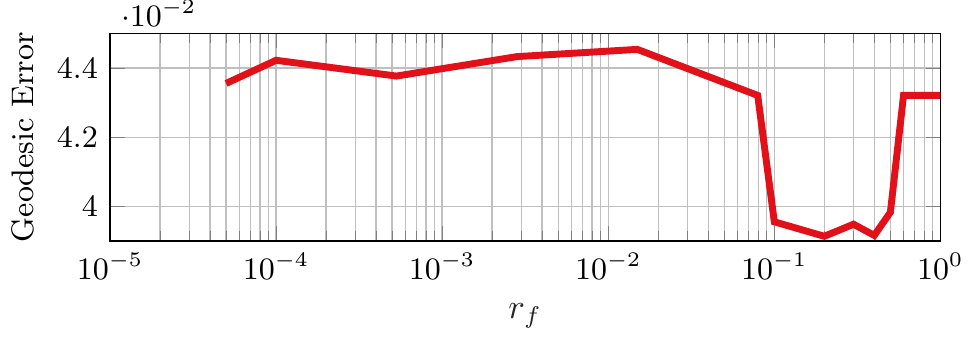}
    \caption{\label{fig:param_study_r_f} Impact of the $r_f$ parameter on the shape matching quality. The mean geodesic error is averaged on the $95$ shape pairs of the TOSCA non-isometric dataset (remeshed to $5$K vertices). Notice how stable our method remains, even for extreme values of $r_f$.}
\end{figure}

\subsubsection{Landmark placement}
In order to study the influence of landmark placement on our method, we conduct the following experiment on $10$ shapes of the TOSCA Isometric dataset (cat category). We consider an increasing number of landmark correspondences, ranging from $3$ to $100$, placed according to four standard surface sampling strategies: (i) random, (ii) euclidean farthest point (iii) geodesic distance farthest point (iv) Poisson disk (as implemented in~\cite{gptoolbox}). The outcome of these experiments is illustrated in Fig.~\ref{fig:sampling_strat_TOSCA_Iso}. The farthest point sampling strategies result in the fastest decrease of the error, Poisson disk is slightly slower and random placement is predictably the slowest. This indicates that our method performs best when the  extremities of the shapes are prioritized for landmark placement. The landmark placement used in the benchmarks of Sec.~\ref{Sec:benchmarks} makes use of this observation (see App.~\ref{Sec:appendix_evaluation_setup_details} for details).

\begin{figure}
    \centering
    \includegraphics[width=\linewidth]{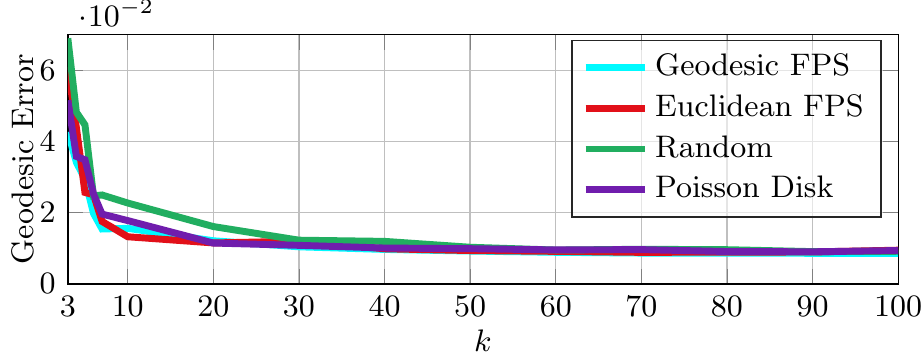}
    \vspace{-5mm}
    \caption{\label{fig:sampling_strat_TOSCA_Iso} Error summary when increasing the number of landmarks $k$ for different surface sampling strategies. The mean geodesic error on $10$ cat shapes of the TOSCA Isometric dataset is reported. ``FPS'' stands for Farthest Point Sampling.}
\end{figure}

To complement the above experiment, we show the variance of our method when initializing two sampling strategies with $3$ different seeds in Fig.~\ref{fig:sampling_strat_TOSCA_non_Iso} on the full TOSCA non-isomtric dataset.

\begin{figure}
    \centering
    \includegraphics[width=\linewidth]{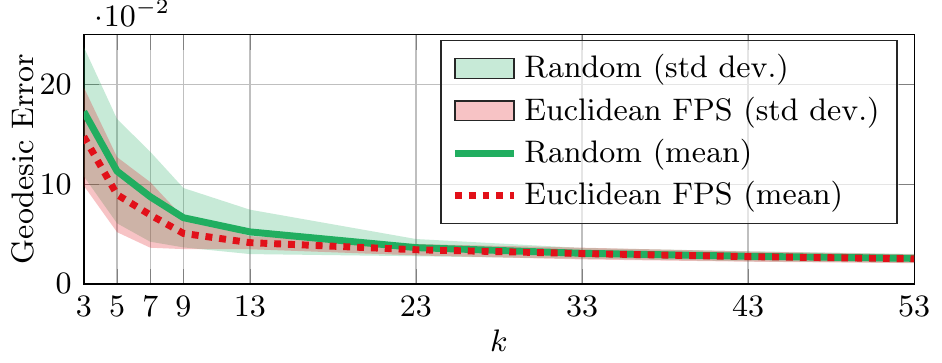}
    \vspace{-5mm}
    \caption{\label{fig:sampling_strat_TOSCA_non_Iso} Error summary when increasing the number of landmarks $k$ for two surface sampling strategies. The mean geodesic error on $95$ shape pairs of the TOSCA non-isomatric dataset with $3$ different seed initializations for each pair is displayed. ``FPS'' and ``std dev.'' respectively stand for Farthest Point Sampling and standard deviation.\vspace{-2mm}}
\end{figure}

\subsubsection{Remeshing invariance}

In order to show that our method remains applicable on shapes with different triangulations, we remesh independently the target pair of each FAUST shape pair and compute the mean geodesic error in Fig.~\ref{fig:remeshing_stability} (left). We additionally experiment with the FAUST ``Wild'' dataset created in~\cite{sharp2020diffusion} to assess invariance to the remeshing proposed by the authors. Fig.~\ref{fig:remeshing_stability} (right) and Tab.~\ref{tab:faust_wild_mean} present the output of this experiment. We observe marginal difference when considering the various remeshing approaches tested, which highlights the insensitivity of the proposed approach to the shape connectivity. Fig.~\ref{fig:faust_wild_quali} illustrates qualitatively the median transfer obtained on this dataset.

\begin{figure}
    \centering
    \includegraphics[width=0.49\columnwidth]{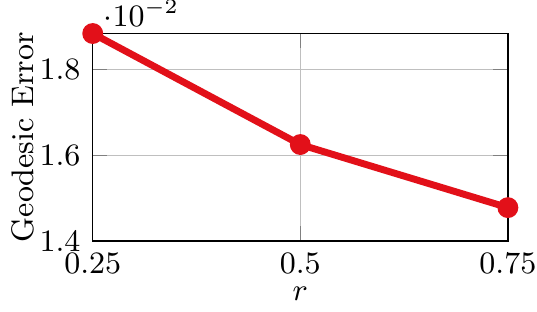}
    \includegraphics[width=0.49\columnwidth]{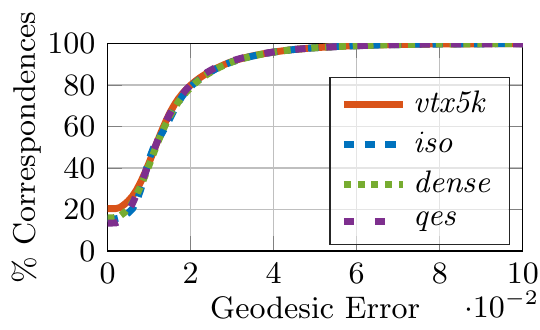}
    \caption{\label{fig:remeshing_stability}\textbf{Left:} remeshing stability when varying the triangle reduction factor $r$ of the target shape. The geodesic error, averaged over $300$ test pairs of the FAUST data set, slightly increases when the target mesh becomes coarse (low value of $r$). \textbf{Right:} stability of our method when performing resmeshings on the FAUST dataset (Remeshed to $5$K vertices and FAUST ``Wild'' (see Sec.~\ref{sec:evaluation_datasets}) ). The geodesic error is measured in mean geodesic distance $\times100$ after normalizing by the geodesic diameter. The mean values, mean execution times and vertex counts for each remeshing is presented in Tab.~\ref{tab:faust_wild_mean}.\vspace{-2mm}}
\end{figure}

\begin{figure}
    \centering
    \includegraphics[width=\columnwidth]{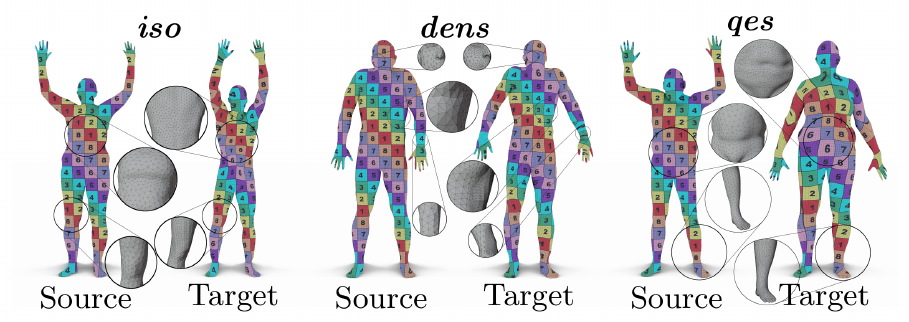}
    \caption{\label{fig:faust_wild_quali}Qualitative illustration of the median map quality obtained with our method on three types of remeshing in the FAUST ``Wild'' dataset (see Sec.~\ref{sec:evaluation_datasets}). Despite the great disparity of the underlying meshes, our method provides smooth transfers.}
\end{figure}

\begin{table}
\centering
\begin{tabular}{ccccc}
\toprule
             & \textit{vtx5k} & \textit{iso}  & \textit{dense}  & \textit{qes}   \\ \hline
Geo. Err.    & $13.7$ & $14.3$ & $14.1$ & $14.2$  \\
$n_v$        & $5001$ & $7117$ & $13399$ & $14002$    \\
Exec. t. (s) & $7.3$ & $8.35$ & $13.75$  & $14.1$  \\
\bottomrule
\end{tabular}
\caption{\label{tab:faust_wild_mean}Stability of our method when performing resmeshings on the FAUST dataset. The geodesic error (geo. err.) is measured in mean geodesic distance $\times100$ after normalizing by the geodesic diameter. The corresponding error curves are displayed in Fig.~\ref{fig:remeshing_stability} (right). The execution time (exec. t.) is also reported, along with the mean number of vertices for each remeshing type ($n_v$).}
\end{table}

\subsection{Benchmarks}\label{Sec:benchmarks}
In this section, we describe the competing state-of-the-art methods that we employ (Sec.~\ref{sec:benchmark_setup}) and present our main results for shape matching (Sec.~\ref{sec:evaluation_results}).

\subsubsection{Setup}\label{sec:benchmark_setup}
We compare our method against three competitors that leverage landmark information to compute correspondences between shapes. The detailed setup for each method, including the landmark placement is provided in App.~\ref{Sec:appendix_evaluation_setup_details}. The competing methods are:
\par
\textbf{Hyperbolic Orbifold Tutte Embeddings (hyperOrb)}~\cite{tutte} constructs a parameterization of each surface by embedding the points to the hyperbolic plane. The surfaces are cut along the input correspondences, which are \textit{de facto} preserved.
\par
\textbf{Weighted Averages (WA)}~\cite{panozzo2013weighted} also defines a parameterization of the input surfaces that preserves landmarks exactly: each point at the surface is expressed as a weighted average of its distance to a set of landmarks.
\par
\textbf{Functional Maps With ZoomOut Refinement (FMap ZO)~\cite{melzi2019zoomout}} computes correspondences between shapes by leveraging a functional basis defined on the source and target shapes. While the method does not allow to retrieve exact correspondence between user-specified landmarks, it constitutes the current state-of-the-art method for isometric shape matching.

\subsubsection{Results}\label{sec:evaluation_results}
In this section, we present our main results on shape matching.
\par
\textbf{Isometric shape matching.} The evaluation on FAUST and TOSCA Isometric are illustrated in Fig.~\ref{fig:FAUST_TOSCA_evaluation}, with averaged errors and runtimes displayed in Tab.~\ref{tab:isometric_datasets_evaluation}. On the FAUST data set, our approach remains competitive with a mean geodesic error of $1.40\times10^{-2}$ and a mean computation time of $8.83$ s. On the TOSCA isometric data set, we obtain a slightly better average geodesic error score than competitors. Qualitatively, our method produces smooth texture transfers on both data sets, as highlighted in Fig.~\ref{fig:iso_median_quali}.

\begin{figure}
    \centering
    \includegraphics[width=0.49\linewidth]{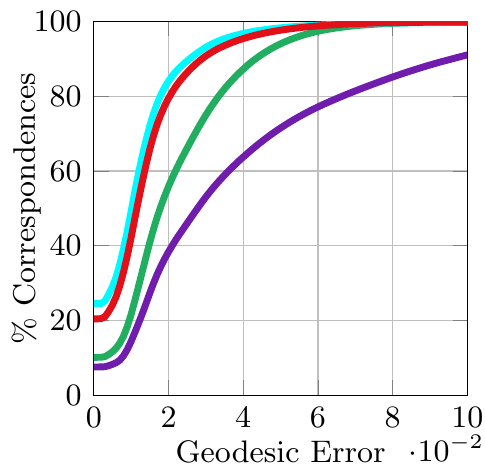}
    \includegraphics[width=0.49\linewidth]{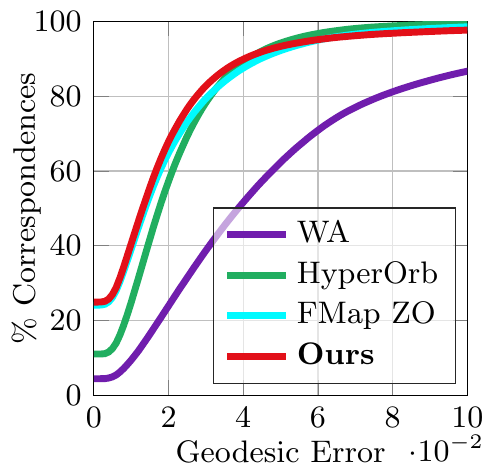}
    \caption{\label{fig:FAUST_TOSCA_evaluation}Error summary on the FAUST (\textbf{left}) and TOSCA Isometric dataset (\textbf{right}). The geodesic error is measured in mean geodesic distance $\times100$ after normalizing by the geodesic diameter.}
\end{figure}

\begin{table}
\centering
\begin{tabular}{cccc} 
\toprule
Method                         & Data Set   & Av. Geo. Err. & Av. Time (in s.) \\ 
\midrule
\multirow{2}{*}{FMap ZO }      & FAUST      &   $\mathbf{1.23\times10^{-2}}$            &  $\mathbf{5.93}$     \\ 

                              & TOSCA Iso. &    $1.95\times10^{-2}$           &   $\mathbf{6.27}$   \\ 
\hline
\multirow{2}{*}{HyperOrb}      & FAUST      &    $2.19\times10^{-2}$           &  $26.8$     \\ 

                              & TOSCA Iso. &    $2.10\times10^{-2}$           &   $10.5$    \\ 
\hline
\multirow{2}{*}{WA}            & FAUST      &     $4.08\times10^{-2}$          &   $59.3$    \\ 

                              & TOSCA Iso. &    $5.26\times10^{-2}$           &   $81.0$    \\ 
\hline
\multirow{2}{*}{\textbf{Ours}} & FAUST      &     $1.40\times10^{-2}$          &   $8.83$    \\ 

                              & TOSCA Iso. &    $\mathbf{1.90\times10^{-2}}$           &    $11.3$   \\
\bottomrule
\end{tabular}
\caption{Quantitative evaluation results on the remeshed FAUST and TOSCA Isometric (TOSCA Iso.) data sets. The average geodesic error (Av. Geo. Err.) and average execution time (Av. Time) on both data sets are displayed for our method and competing approaches.}
\label{tab:isometric_datasets_evaluation}
\end{table}

\begin{figure}
\centering
    \begin{overpic}[height=4cm, trim=0 10 0 -5, clip]{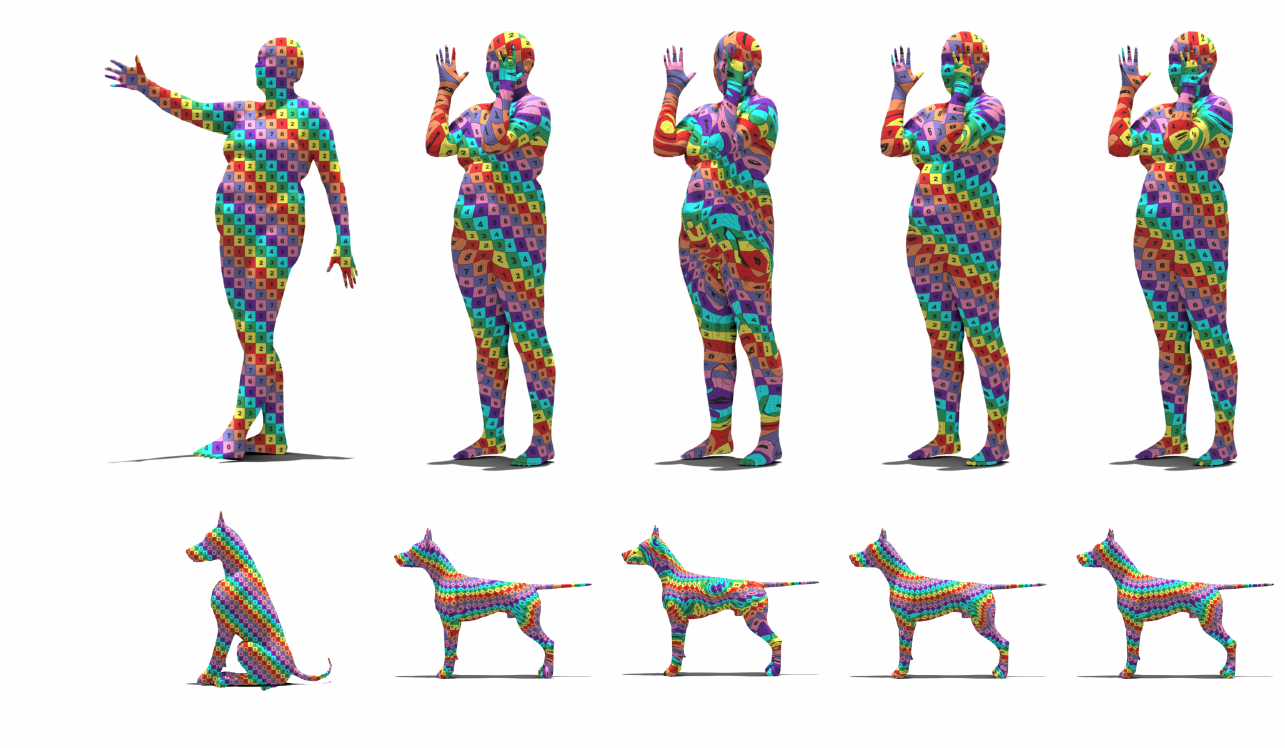}%
    \put(0,38) {FAUST}%
    \put(0,7) {TOSCA}%
    \put(13,55) {Source}%
    \put(29,55) {HyperOrb}%
    \put(53,55) {WA}%
    \put(65,55) {FMap ZO}%
    \put(88,55) {\textbf{Ours}}%
    \end{overpic}
\caption{\label{fig:iso_median_quali}Qualitative evaluation of our method and competing approaches on isometric shapes. The first row corresponds to shapes from the FAUST data set. The bottom row consists of shapes from the TOSCA isometric data set. The shape pair is selected such that the geodesic error of our method is \textbf{median} over the dataset. The best and worst cases are illustrated in App.~\ref{sec:appendix_additional_quali}.}
\end{figure}

\par
\textbf{Non-isometric shape matching.} We run an evaluation of our method on the TOSCA non-isometric and the SHREC'20 datasets (Fig.~\ref{fig:TOSCA_NONISO_SHREC20_evaluation}). The mean error values and timings are showed in Tab.~\ref{tab:NONisometric_datasets_evaluation}. In this challenging setup, our method has the best results in terms of mean geodesic error, while being the second best in terms of computation time. Fig.~\ref{fig:noniso_median_quali} presents a qualitative evaluation using a texture transfer on a pair of shapes for each data set.

\begin{figure}
    \centering
    \includegraphics[width=0.45\linewidth]{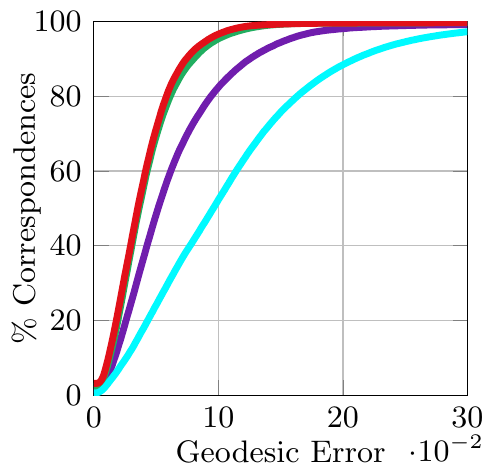}
    \includegraphics[width=0.45\linewidth]{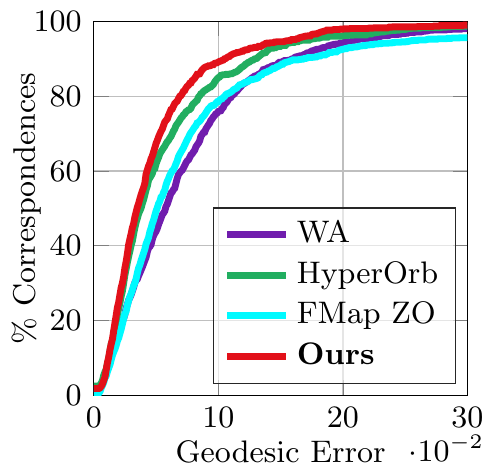}
    \caption{\label{fig:TOSCA_NONISO_SHREC20_evaluation}Error summary on the TOSCA non-isometric (\textbf{left}) and on the SHREC'20 lores dataset (\textbf{right}).}
\end{figure}

\begin{table}
\centering
\begin{tabular}{cccc} 
\toprule
Method                         & Data Set   & Av. Geo. Err. & Av. Time (in s.) \\ 
\midrule
\multirow{2}{*}{FMap ZO }      & TOSCA n-i.      &   $1.10\times10^{-1}$            &  $\mathbf{7.78}$     \\ 
                              & SHREC'20 &    $7.86\times10^{-2}$           &   $\mathbf{27.9}$   \\ 
\hline
\multirow{2}{*}{HyperOrb}      & TOSCA n-i.      &   $4.33\times10^{-2}$            &  $17.8$     \\
                              & SHREC'20 &    $5.78\times10^{-2}$           &   $270$    \\ 
\hline
\multirow{2}{*}{WA}            & TOSCA n-i.      &   $6.50\times10^{-2}$            &  $79.7$     \\ 
                              & SHREC'20 &    $7.62\times10^{-2}$           &   $140$    \\ 
\hline
\multirow{2}{*}{\textbf{Ours}} & TOSCA n-i.      &   $\mathbf{4.11\times10^{-2}}$            &  $13.5$     \\
                              & SHREC'20 &    $\mathbf{5.09\times10^{-2}}$           &    $63.8$   \\
\bottomrule
\end{tabular}
\caption{Quantitative evaluation results on the TOSCA non-isometric (n-i.) and the SHREC'20 lores (without partial shapes) data sets. The average geodesic error (Av. Geo. Err.) and average execution time (Av. Time) on both data sets are displayed for competing approaches and our method.}
\label{tab:NONisometric_datasets_evaluation}
\end{table}

\begin{figure*}
\centering
    \begin{overpic}[width=\textwidth, trim=0 10 0 10, clip]{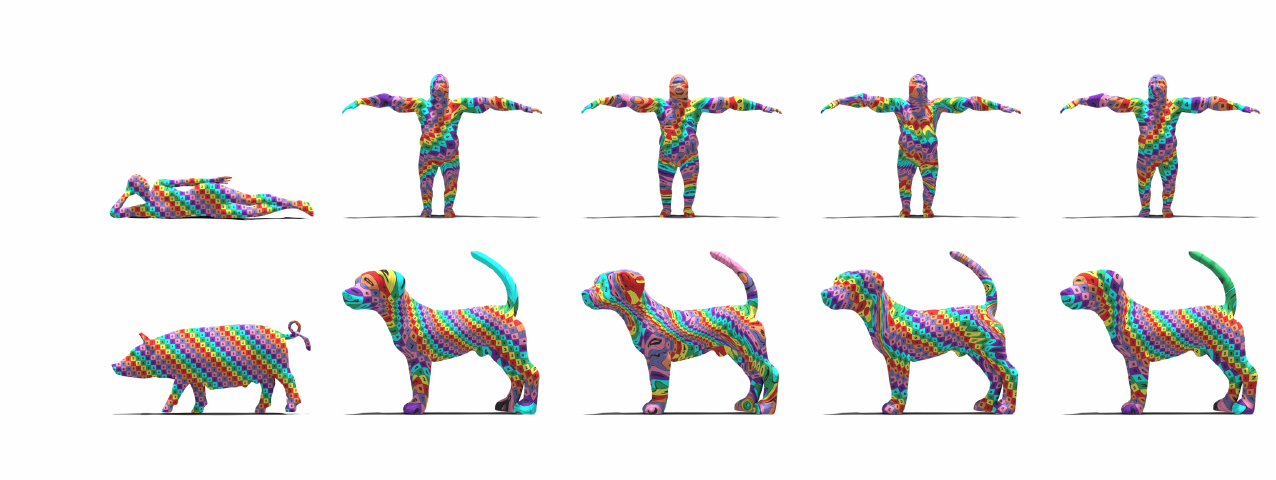}%
    \put(0,23) {TOSCA}%
    \put(0,5) {SHREC'20}%
    \put(13,30) {Source}%
    \put(31,30) {HyperOrb}%
    \put(52,30) {WA}%
    \put(69,30) {FMap ZO}%
    \put(89,30) {\textbf{Ours}}%
    \end{overpic}
\caption{\label{fig:noniso_median_quali}Qualitative evaluation of our method and competing approaches on non-isometric shapes. The first row corresponds to shapes from the TOSCA non-isometric data set. The bottom row consists of shapes from the SHREC'20 lores data set. Each shape pair is selected such that the geodesic error of our method is \textbf{median} over the dataset. The best and worst cases are illustrated in App.~\ref{sec:appendix_additional_quali}.}
\end{figure*}

\par
\textbf{SHREC'19 benchmark.} The quantitative evaluation is reported in Fig.~\ref{fig:SHREC19_evaluation}, with the associated averaged geodesic errors on the right of the figure. Our method obtains the best mean geodesic error score for this difficult benchmark. In addition, a qualitative evaluation via texture transfer is depicted in Fig.~\ref{fig:SHREC19_median_quali}. Our method's strong performance on this dataset is indicative of its stability and applicability across diverse changes in shape topology, such as the introduction of small holes. This is a general feature of the functional maps methods, which our approach inherits.

\begin{figure}
  \begin{minipage}[b]{0.45\linewidth}
    \centering
    \includegraphics[width=\linewidth]{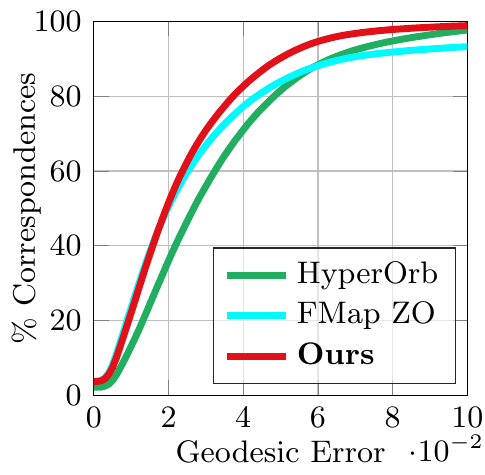}
    \par\vspace{0pt}
  \end{minipage}%
  \begin{minipage}[b]{0.5\linewidth}
    \centering%
    \begin{tabular}{cccc}%
        \toprule%
        Method                         & Av. Geo. Err. \\
        \midrule%
        FMap ZO      & $3.84\times10^{-2}$            \\
        \hline%
        HyperOrb      & $3.26\times10^{-2}$        \\
        \hline%
        \textbf{Ours} & $\mathbf{2.48\times10^{-2}}$       \\
        \bottomrule%
    \end{tabular}
    \par\vspace{22pt}
\end{minipage}
\caption{\label{fig:SHREC19_evaluation}Error summary on $165$ shapes of the SHREC'19 data set. The average geodesic error (Av. Geo. Err.) is displayed for our method and competing approaches.}
\end{figure}

\begin{figure}
\centering
    \begin{overpic}[height=3.25cm, trim=20 10 0 0, clip]{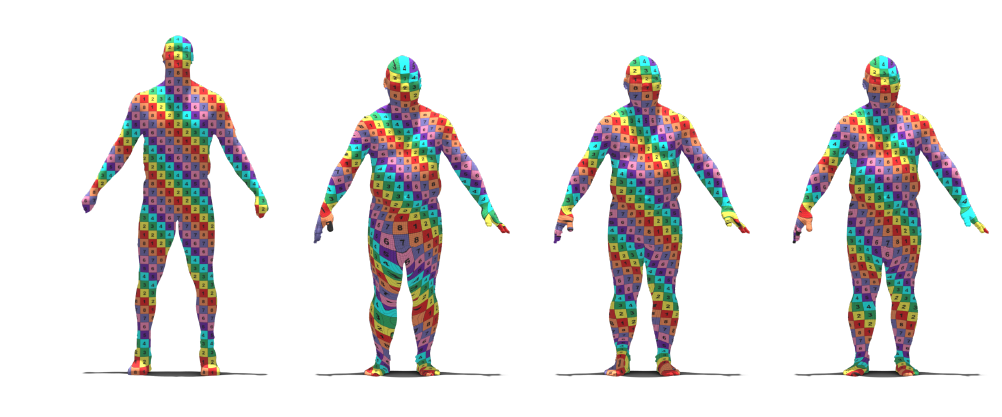}%
    \put(7,38) {Source}%
    \put(28,38) {HyperOrb}%
    \put(54,38) {FMap ZO}%
    \put(83,38) {\textbf{Ours}}%
    \end{overpic}
\caption{\label{fig:SHREC19_median_quali}Qualitative evaluation of our method and competing approaches on a shape pair from the SHREC'19 data set, selected such that the geodesic error of our method is \textbf{median} over the dataset. The best and worst cases are illustrated in App.~\ref{sec:appendix_additional_quali}.}
\end{figure}

%% file: Sections/s08_Conclusion.tex
\section{Conclusion, Limitations and Outlook}
We have proposed an efficient functional maps-based shape matching approach that promotes conformal maps and exactly preserves landmark correspondences. This was achieved via the introduction of a novel functional basis and an energy promoting bijective conformal maps. The efficiency of our solution comes from an adaptation of the ZoomOut procedure \cite{melzi2019zoomout,DiscreteSolver} using our energy and novel basis. The resulting method exhibits state-of-the-art performance on non-isometric benchmark datasets and near state of the art performance on isometric ones.

Recall, however, that our usage of the ZoomOut procedure was not fully principled. Indeed, we needed to make some approximations in order to use Lemma \ref{Lemma:ConvertToNN}, which converts certain optimization problems into nearest-neighbor searches. The quality of our results indicates that our approximations were justified, suggesting that Lemma \ref{Lemma:ConvertToNN} could likely be rigorously extended to suit our needs. In fact, extending Lemma \ref{Lemma:ConvertToNN} would be of general interest to the functional maps community, as it would enable the efficient minimization of various other energies.

The construction of our landmark-adapted basis required us to upgrade the landmarks to proper boundaries. We did so by cutting out small disks centered at the landmarks, resulting in the introduction of landmark circles. The landmark circles offer an intriguing possibility that we have not explored here. Namely, one could augment landmark correspondence to include a user-specified matching of the landmark circles. This could allow for greater semantic or artistic control of the resulting map. Our initialization procedure of Sec.~\ref{Sec:InitialGuess} can be seen as an automated implementation of a similar idea.\\

Furthermore, since our present work has demonstrated the fruitfulness of landmark-adapted bases, it is natural to ask whether better performance can be achieved by improving upon basis construction. In particular, we have noted that the Dirichlet-Steklov eigenfunctions have their amplitude intensely concentrated near the landmark circle equipped with the Steklov boundary condition (see Fig.~\ref{Fig:AnnulusBasisComparison}). It seems likely that an analogous basis with less concentrated functions could be better suited to describe the behavior of the functional map near the landmarks. Notice that this dovetails with the idea of user-specified landmark circle correspondence, as the user-provided information would have impact further away from the landmarks.

%% file: Sections/s09_Acknowledgements.tex
\section{Acknowledgements}
The authors would like to acknowledge the anonymous reviewers for their helpful feedback and suggestions. Parts of this work were supported by the ERC Starting Grants No. 758800 (EXPROTEA), the ANR AI Chair AIGRETTE, the Swiss National Science Foundation (SNSF) under project number 188577 and the Association Nationale de la Recherche et de la Technologie (ANRT) via the Convention industrielle de formation par la recherche (CIFRE) grant No. 2019/0433. Finally, the authors wish to thank Prof. Alexandre Girouard for his assistance in navigating the literature on the Steklov eigenproblem, Jing Ren and Simone Melzi for providing tools for shape analysis via functional maps and the TOSCA/FAUST/SHREC'19 datasets, Nicholas Sharp for releasing the ``FAUST Wild'' dataset and Patrick Schmidt for computing a baseline to his work~\cite{schmidt2020inter}.

%% file: Sections/s10_Discretization.tex
\section{Weak Form of the Dirichlet-Steklov Eigenproblem} \label{sec:DSWeakForm}

In this appendix, we derive the weak form of the Dirichlet-Steklov eigenproblem (Eq. \eqref{Eq:WeakDirichletSteklov}), in which it becomes very similar to the weak form of the more familiar Laplacian eigenproblem. For sufficiently smooth functions $f$ and $u$, Stokes' theorem implies that

\begin{equation}
	\begin{aligned}
			\int_{\mathcal{M}} f \left(\Delta u \right) ~d\mathcal{M}  = &~ &&\int_{\mathcal{M}} \nabla f \cdot \nabla u ~d\mathcal{M}\\
		&- &&\int_{\partial \mathcal{M}} f \left( \partial_n u \right) ~d\left( \partial \mathcal{M} \right)~.
	\end{aligned}
\end{equation}

\noindent Applying this to a $u_i$ satisfying the Dirichlet-Steklov eigenproblem (Eq. \eqref{Eq:DirichletSteklov}) and a smooth test function $f$ vanishing on $\mathcal{D}$ yields

\begin{equation}
	\begin{aligned}
	\int_{\mathcal{M}} \nabla f \cdot \nabla u_i  ~d \mathcal{M} &= &&~ &&\cancelto{0}{~~\int_{\mathcal{M}} f \left(\Delta u_i \right) ~d \mathcal{M} ~~}\\
	&~ &&+ &&\int_{\partial \mathcal{M}} f \left( \partial_n u_i \right) ~d(\partial \mathcal{M})\\
	&= &&~ &&\cancelto{0}{~~\int_{\mathcal{D}} f \left( \partial_n u_i \right) ~d(\partial \mathcal{M}) ~~}\\
	&~ &&+ &&\int_{S} f \left( \partial_n u_i \right) ~d(\partial \mathcal{M})~,
	\end{aligned}
\end{equation}

\noindent where the first cancellation arises from the harmonicity of $u_i$ and the second one from $f$ vanishing on $\mathcal{D}$. Finally, using the third line of Eq. \eqref{Eq:DirichletSteklov} results in the weak form of the Dirichlet-Steklov problem:

\begin{equation} \label{eq:DS_weak_form_appendix}
	\int_{\mathcal{M}} \nabla f \cdot  \nabla u_i  ~d\mathcal{M} = \sigma_i \int_{S} f u_i ~d(\partial \mathcal{M})~.
\end{equation}

\noindent This can be readily discretized on triangle meshes, as discussed in App.~\ref{sec:DiscretizationEigenproblems}.

\section{Discretization of the Eigenproblems} \label{sec:DiscretizationEigenproblems}

In this appendix, we briefly discuss the discretization on triangle meshes of the eigenproblems used in our approach.

\paragraph*{Discretization of the Dirichlet Laplacian eigenproblem}
We begin with the familiar Dirichlet Laplacian eigenproblem (Eq. \eqref{Eq:DirichletLaplaceBeltrami}). We discretize this problem using the well-known cotangent scheme (piecewise-linear finite elements). The problem then becomes

\begin{equation} \label{eq:DiscreteLB_appendix}
    \begin{aligned}
        &W^{_{\mathcal{M}}} \psi_i = \lambda_i A^{_{\mathcal{M}}} \psi_i~,\\
        &\psi_i \big|_{_{\partial \mathcal{M}}} = 0~,
    \end{aligned}
\end{equation}

\noindent where $W^{_{\mathcal{M}}}$ denotes the so-called cotangent Laplacian and $A^{_{\mathcal{M}}}$ denotes the lumped mass matrix. See \cite{bunge2020polygon}, among many others, for a definition of these objects.

\paragraph*{Discretization of the Dirichlet-Steklov eigenproblem}

We use piecewise linear finite elements to discretize the weak form of the Dirichlet-Steklov eigenproblem (Eq. \eqref{eq:DS_weak_form_appendix}). The left-hand side of the expression becomes the familiar cotangent Laplacian, denoted by $W^{_\mathcal{M}}$. The discretization of the integral on the right-hand side requires a mass matrix defined \emph{strictly on the boundary}. Similarly to the mass matrix used in the Laplacian eigenproblem, it can be discretized either according to a piecewise-linear finite element scheme, or as a lumped mass matrix. Regardless of the chosen discretization, we call this mass matrix $S^{_\mathcal{M}}$. Note that $S^{_\mathcal{M}}$ is of the same size as $W^{_\mathcal{M}}$.

We begin by the lumped discretization. The boundary is one-dimensional. Thus, a vertex $p \in \partial \mathcal{M}$, has (at most) two neighbors that are also in $\partial \mathcal{M}$, which we denote $p-1$ and $p + 1$. The length of the edges $(p-1,p)$ and $(p,p+1)$ are denoted $r_{p-1}$ and $r_{p+1}$, respectively. The lumped Steklov mass matrix is given by

\begin{equation}
S^{_\mathcal{M}}_{pq}=
\begin{cases}
\frac{1}{2}(r_{p-1} + r_{p+1})~~~&,~~~ p = q \text{ and }p,q \in \partial \mathcal{M}\\
0~~~&,~~~\text{elsewhere.}
\end{cases}
\end{equation}

The non-lumped mass matrix is computed from a piecewise linear finite element discretization on the boundary. This discretization corresponds to the restriction of the piecewise linear finite elements of the mesh to the boundary edges. Whenever vertices $p$ and $q$ are distinct endpoints of the same edge, we write $p \sim q$. The length of the edge connecting $p$ and $q$ is denoted $r_{pq}$. After a straightforward computation which we omit, the non-lumped Steklov mass matrix is given by

\begin{equation}
S^{_\mathcal{M}}_{pq}=
\begin{cases}
\frac{1}{3}(r_{p-1} + r_{p+1})~~~&,~~~ p = q \text{ and }p,q \in \partial \mathcal{M}\\
\frac{1}{6} r_{pq}~~~&,~~~ p \sim q \text{ and }p,q \in \partial \mathcal{M}\\
0~~~&,~~~\text{elsewhere.}
\end{cases}
\end{equation}

\noindent In sum, no matter the version of $S^{_\mathcal{M}}$ chosen, the discretization of the Dirichlet-Steklov problem becomes

\begin{equation}
\begin{aligned}
W^{_\mathcal{M}} u_i &= \sigma_i S^{_\mathcal{M}} u_i~,\\
u_i \big|_{D} &= 0~,
\end{aligned}
\end{equation}

\noindent which is quite similar to the more familiar Laplacian eigenvalue problem with Dirichlet boundary conditions (Eq. \eqref{eq:DiscreteLB_appendix}).\\

\paragraph*{ A Word of Warning} As a final note on the discretization of the considered eigenproblems, we would like to warn the reader of a small issue one may encounter when numerically solving them. Recall that we want the Dirichlet-Steklov eigenfunctions to be normalized with respect to the boundary mass matrix $S^{_\mathcal{M}}$. Solvers for generalized eigenvalue problems, such as Matlab's \texttt{eigs} routine, which we use in our implementation, will typically do so automatically. However, according to our observations, sometimes this automated process will not happen. This seems to be related to the fact that $S^{_\mathcal{M}}$ is a positive semi-definite matrix rather than a positive definite one. Thus, one needs to explicitly normalize the solutions with respect to $S^{_\mathcal{M}}$. In fact, we suggest explicitly normalizing even the Laplacian eigenfunctions, despite the fact that there the mass matrix $A^{_\mathcal{M}}$ is positive definite on (good quality) triangle meshes. Indeed, $A^{_\mathcal{M}}$ can fail to be positive-definite on pathological inputs. Consider for instance an otherwise good mesh with an isolated vertex belonging to no triangle. Functions vanishing everywhere except on said vertex have norm $0$ with respect to $A^{_\mathcal{M}}$, despite being nonzero.

%% file: Sections/s11_DiscreteLandmarkCircle.tex
\section{Boundary Circles on Triangle Meshes} \label{sec:DiscreteLandmarkCircle}

In Sec.~\ref{Sec:LandmarkAdaptedBasis}, small disks centered at the landmarks are removed in order to create new boundaries for the shapes under study. Here, we describe in detail how this is achieved on triangle meshes. Crucially, we do not want to unduly disturb the geometry of the shapes. In order to achieve this we construct the new boundaries entirely within the triangles adjacent to the landmarks.\\

Let's say that we are constructing the boundary circle for the landmark $\gamma_i$. We begin by selecting the radius $r_i$ of the disk to be removed. This is done by finding the length $s_i$ of the shortest edge connected to $\gamma_i$. The minimum is taken over both shapes, which are scaled to be of identical surface area and thus of comparable size. Then, we set $r_i = r_f \cdot s_i$, where $r_f \in (0,1)$ is a user-set parameter. The (surprisingly low) impact of this parameter is studied in Sec.~\ref{Sec:r_f_study}.\\

We are now ready to construct the boundary $\Gamma_i$. \emph{This process is best understood by looking at its illustration in Fig.~\ref{fig:refinement_procedure}}. First, we split each triangle adjacent to the landmark into $n_s$ wedges of equal angle, which introduces $n_s-1$ new vertices at the opposite edge of the original triangle, as well as edges connecting them to the landmark. Then, we introduce $n_s+1$ new vertices situated on the new edges at a distance $r_i$ away from the landmark $\gamma_i$. We then connect these vertices in a way that creates an approximation of a sector of a disk of radius $r_i$. Doing so produces $n_s$ quadrilaterals in the part of the original triangle far from the landmark. We split those quadrilaterals into triangles along their diagonals. This concludes the refinement of the triangles adjacent to the landmark. It remains to refine the triangles adjacent to them across the edges opposite to the landmark. There, the common edges between the triangles contains $n_s-1$ new vertices. On each triangle, we connect these new vertices to the original vertex not on the common edge. This concludes the refinement process. Note that all of the new triangles are contained within the original ones. An example of a mesh with landmark circles constructed in this manner is shown in Fig.~\ref{Fig:SphereWproducts}.\\

The construction of the boundaries associated to different landmarks is done sequentially over the landmarks. This requires some additional care if the landmarks are placed too close to each other. Indeed, during the construction of $\Gamma_i$, new faces are created in what was originally the $2-$ring neighborhood of the landmark $\gamma_i$. Thus, if a different landmark $\gamma_j$ is closer than $4$ rings away from $\gamma_i$, there will be overlap between the newly created mesh faces. The resulting mesh will then be dependent upon the order in which the boundary circles $\Gamma_i$ and $\Gamma_j$ are created. In the present paper, we avoid this issue by disallowing such landmark placement. If such landmark placement becomes necessary in a given application, we suggest locally refining the mesh via, say, $\sqrt{3}-$subdivision \cite{kobbelt2003} such that the landmarks are no longer closer than $4$ triangle rings from one another. We do not pursue this here.

\begin{figure}
    \centering
    \includegraphics[width=\linewidth]{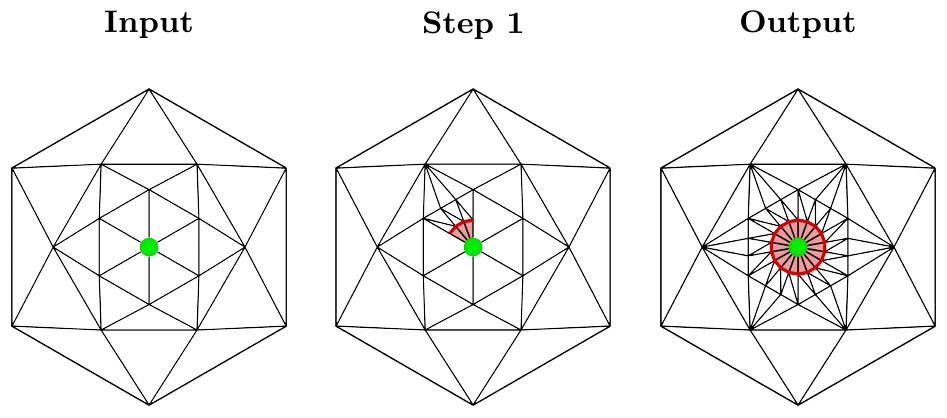}
    \caption{\label{fig:refinement_procedure}Illustration of the creation of a landmark boundary. The landmark position is indicated by a green dot. The triangles composing the landmark disk are shown in light red. The boundary circle is highlighted as a red line. Note that a gap of connectivity appears when creating the boundary around the landmark. This gap is closed when the process finishes producing the boundary.}
\end{figure}

%% file: Sections/s12_ProofOfLemmaOrthogonalDecomposition.tex
\section{Proof of Lemma~\ref{Lemma:WorthogonalSplit} and Discussion on its Meaning} \label{Sec:AppendixProofOfSplit}

\begin{replemma}{Lemma:WorthogonalSplit}
The function space $W(\mathcal{M})$ admits the following decomposition:
\begin{equation}
W(\mathcal{M}) = \mathcal{G(M)} \operp \overline{\left( \bigoplus_{j=1}^{k} \mathcal{H}_j (\mathcal{M}) \right)} ~,
\end{equation}
where $\oplus$ denotes direct sums and $\operp$ denotes orthogonal direct sums.
\end{replemma}
\begin{proof}
Recall that, by construction, $W(\mathcal{M})$ is the completion of smooth functions modulo constants with respect to the Dirichlet form. Thus, we begin our analysis on smooth functions.

Let $u$ be smooth and $W(\mathcal{M})$-orthogonal to all of the Dirichlet-Laplacian eigenfunctions $\{\psi_i\}_{i=1}^{\infty}$. Then, by Stokes' theorem,

\begin{equation}
\begin{aligned}
0 	&= \int_{\mathcal{M}} \nabla \psi_i \cdot \nabla u ~d\mathcal{M}\\
	&= \int_{\mathcal{M}} \psi_i \left(\Delta u \right) ~d\mathcal{M} + \cancelto{0}{~~\int_{\partial \mathcal{M}} \psi_i \left( \partial_n u \right) ~d(\partial \mathcal{M})~~}~,
\end{aligned}
\end{equation}

\noindent where the cancellation results from $\psi_i$ vanishing at the boundary. Since the $\{\psi_i\}_{i=1}^{\infty}$ form an orthogonal basis for $L_2(\mathcal{M})$, this implies that $\Delta u = 0$. Thus, smooth functions can be $W(\mathcal{M})$-orthogonally decomposed into a part that lies in $\mathcal{G(M)}$ (the closed span of $\{\psi_i\}_{i=1}^{\infty}$) and a harmonic part.\\

Recall that each $\mathcal{H}_j(\mathcal{M})$ spans harmonic functions that vanish on all landmark boundaries, but the $j^{th}$ one. Since harmonic functions are uniquely determined by their values at the boundaries, the harmonic part of $u$ can be naturally expressed as an element in $\oplus_{j=1}^{k} \mathcal{H}_j (\mathcal{M})$.\\

Since $W(\mathcal{M})$ is complete by construction and smooth functions are dense in $W(\mathcal{M})$, the desired result is achieved by taking the closure of the subspaces.
\end{proof}

Notice that in the above lemma, the subspaces $\mathcal{H}_j(\mathcal{M})$ are not marked as $W(\mathcal{M})$-orthogonal. Indeed, by Stokes' theorem,

\begin{equation} \label{Eq:CrossLandmarkInnerProduct}
\begin{aligned}
	\langle u_i^{(p)} , u_l^{(q)} \rangle_{_{W(\mathcal{M})}} & &&= &&  &&\int_{\mathcal{M}} \nabla u_i^{(p)} \cdot \nabla u_l^{(q)} ~d\mathcal{M}\\
	& &&= && &&\cancelto{0}{~~\int_{\mathcal{M}} u_i^{(p)} \left( \Delta u_l^{(q)} \right) ~d\mathcal{M}~~}\\
	& && &&+ &&\int_{\partial \mathcal{M}} u_i^{(p)} \left( \partial_n u_l^{(q)} \right) ~d (\partial \mathcal{M})\\
	& &&= &&\sum_{\mu = 1}^{k} &&\int_{\Gamma_\mu} u_i^{(p)} \left( \partial_n u_l^{(q)} \right) ~d(\partial \mathcal{M})\\
	& &&= &&  &&\int_{\Gamma_p} u_i^{(p)} \left( \partial_n u_l^{(q)} \right) ~d(\partial \mathcal{M})
\end{aligned}
\end{equation}

\noindent The above expression yields different results depending on whether $p$ and $q$ coincide or not. We begin by considering $p=q$.

\begin{equation}
\begin{aligned}
\langle u_i^{(p)} , u_l^{(p)} \rangle_{_{W(\mathcal{M})}} &= \int_{\Gamma_p} u_i^{(p)} \left( \partial_n u_l^{(p)} \right) ~d(\partial \mathcal{M}) \\
&= \sigma_{l}^{(p)}   \int_{\Gamma_p} u_i^{(p)}  u_l^{(p)}  ~d(\partial \mathcal{M}) \\
&= \sigma_{l}^{(p)}  \delta_{il}~.
\end{aligned}
\end{equation}

\noindent Here $\delta_{il}$ denotes the Kronecker delta. Thus, for every $p$, the Dirichlet-Steklov basis $\{u_i^{(p)}\}_{i=1}^{\infty}$ is composed of $W(\mathcal{M})$-orthogonal functions. Notice that said eigenfunctions can be $W(\mathcal{M})$-normalized by dividing them by the square root of the corresponding eigenvalue.\\

Now, consider $p \neq q$. In that case, Eq. \eqref{Eq:CrossLandmarkInnerProduct} can no longer be evaluated by substituting the Dirichlet-Steklov eigenvalue for the normal derivative, as it is evaluated on the wrong boundary component. Moreover, the normal derivative $\partial_n u_l^{(q)}$ has no reason to vanish on $\Gamma_p$, which implies that the subspaces spanned by $\{u_i^{(p)}\}_{i=1}^{\infty}$ and $\{u_i^{(q)}\}_{i=1}^{\infty}$ are not $W(\mathcal{M})$-orthogonal.\\

%% file: Sections/s13_Appendix.tex
\section{Proof of Lemma~\ref{Lemma:StructureF}} \label{Sec:ProofOfStructureLemma}

\begin{replemma}{Lemma:StructureF}[Structure of $F_{_{\mathcal{MN}}}$]
Let $F_{_{\mathcal{MN}}}: W(\mathcal{M}) \to W(\mathcal{N})$ be the pullback of a conformal diffeomorphism that preserves the landmark circles. Then, $F_{_{\mathcal{MN}}}$ maps
\begin{enumerate}
	\item $\mathcal{G(M)}$ to $\mathcal{G(N)}$,
	\item $\mathcal{H}_j(\mathcal{M})$ to $\mathcal{H}_j(\mathcal{N})$ for all $j$.
\end{enumerate}
\end{replemma}
\begin{proof}
Since $\varphi: \mathcal{N} \to \mathcal{M}$ is a diffeomorphism, we can express everything on the surface $\mathcal{N}$. Thus, instead of thinking of $\mathcal{M}$ as separate manifold, we treat $\mathcal{N}$ as being equipped with two Riemannian metrics: its original metric $g^{_{\mathcal{N}}}$ and the pullback metric $g^{_{\mathcal{M}}}$. In this representation, the pullback acts as the identity. In particular, this means that $F_{_{\mathcal{MN}}}: W(\mathcal{M}) \to W(\mathcal{N})$ is a bounded operator.\\

 Since $\varphi$ is conformal, there exists a positive function $\omega$ such that $g^{_{\mathcal{M}}} = \omega g^{_{\mathcal{N}}}$ and $\Delta^{_{\mathcal{M}}} = (1/\omega) \Delta^{_{\mathcal{N}}}$. Let $u$ be a harmonic function on $\mathcal{M}$. Then, $\Delta^{_{\mathcal{N}}} F_{_{\mathcal{MN}}} u  = \omega \Delta^{_{\mathcal{M}}} u = 0$. Thus, $F_{_{\mathcal{MN}}}$ maps harmonic functions to harmonic functions. Furthermore, since $F_{_{\mathcal{MN}}}$ is the pullback of a map that preserves the landmark circles, it maps smooth functions that vanish on all landmark circles of $\mathcal{M}$ but $\Gamma_j^{_{\mathcal{M}}}$ to smooth functions that vanish on all landmark circles of $\mathcal{N}$ but $\Gamma_j^{_{\mathcal{N}}}$ and so for any fixed $j$. Statement $2.$ then follows from the completeness of $W(\mathcal{M})$ and $W(\mathcal{N})$ and the boundedness of $F_{_{\mathcal{MN}}}$ by taking the closure of the relevant subspaces.
 
 Now consider $f \in \mathcal{G(M)}$. By Lemma~\ref{Lemma:WorthogonalSplit}, for all harmonic $u$,
 
 \begin{equation}
 \langle u, f \rangle_{_{W(\mathcal{M})}} = 0~.
 \end{equation}

\noindent By Theorem~\ref{Th:ConformalCharacterization}, the conformality of $\varphi$ allows us to replace the inner product on $W(\mathcal{M})$ with that on $W(\mathcal{N})$ up to the introduction of two functional maps:

\begin{equation}
\langle F_{_{\mathcal{MN}}} u, F_{_{\mathcal{MN}}} f \rangle_{_{W(\mathcal{N})}} = 0~.
\end{equation}

\noindent Since $F_{_{\mathcal{MN}}}$ maps harmonic functions to harmonic functions and is invertible, $F_{_{\mathcal{MN}}} u$ can be any desired harmonic function of $\mathcal{N}$. Thus, $F_{_{\mathcal{MN}}} f$ is $W(\mathcal{N})$-orthogonal to harmonic functions of $\mathcal{N}$, that is $F_{_{\mathcal{MN}}} f \in \mathcal{G(N)}$. This concludes the proof of statement $1.$
\end{proof}

\section{Definition of the Dirichlet Energy}\label{Sec:appendix_dirichlet_energy}

Consider a smooth map $\varphi: \mathcal{M} \to \mathcal{N}$ between two smooth Riemannian manifolds. The Dirichlet energy of the map is 

\begin{equation}
    D(\varphi) = \frac{1}{2} \int_\mathcal{M} \| d \varphi \|^2 d\mathcal{M},
\end{equation}

\noindent where $d \varphi$ is the differential of $\varphi$. Informally speaking, the Dirichlet energy measures the oscillation of the map $\varphi$. The larger the energy, the more oscillatory the map. Maps minimizing the Dirichlet energy are known as harmonic maps. Such maps are a simultaneous generalization of geodesics and harmonic functions. See \cite{jost2008riemannian} for the relevant theory.\\

In the discrete setting, we use the same method as in \cite{ezuz2019reversible} to compute the Dirichlet energy. Namely, the expression becomes

\begin{equation}\label{eq:DiscreteDirichletEnergy}
    D(\varphi) = \frac{1}{4} \sum_{(u,v) \in \mathcal{E}(\mathcal{M})} w^{_{\mathcal{M}}}_{uv} D^2_{_{\mathcal{N}}} \left(\varphi(u), \varphi(v) \right)~,
\end{equation}

\noindent where $\mathcal{E}(\mathcal{M})$ denotes the edges of the mesh $\mathcal{M}$, $w^{_{\mathcal{M}}}_{uv}$ denotes the cotangent weight of the edge $(u,v)$ and $D^2_{_{\mathcal{N}}}(\cdot, \cdot)$ is the matrix of square geodesic distances on $\mathcal{N}$.

\section{Additional Experiments} \label{Sec:appendix_additional_experiments}

\subsection{Analysis of Alternative Initialization Methods}\label{Sec:appendix_alternative_initialization_comparison}

The iterative optimization procedure detailed in Sec.~\ref{Sec:ConversionToNearestNeighbor} requires as an input an initial guess of the functional map. In Sec.~\ref{Sec:InitialGuess}, we thus introduce an initialization procedure for this initial guess based on the landmark correspondence and the normal derivatives of certain landmark-dependent harmonic functions. In this section we compare this approach to two alternatives.\\

For the purposes of this discussion, the approach of Sec.~\ref{Sec:InitialGuess} shall be referred to as the ``normal derivatives'' method. The two alternatives described below will be termed ``trivial'' and the ``conformal energy'', for reasons that should soon become apparent.\\

The landmark circles can be seen as lists of vertices ordered counter-clockwise as seen from outside the shape. The choice of the first element of this list carries no particular meaning and is left to the whims of the indexing of the faces of the mesh. Thus, the first elements of two corresponding boundary circles need not match. The ``trivial'' approach consists in assuming that the first elements of the boundary circles do indeed match. This correspondence is then proportionally extended to the rest of the landmark circle.\\

The ``conformal energy'' approach stems from the observation that mapping the landmark circles $\Gamma_i^{_\mathcal{N}} \to \Gamma_i^{_\mathcal{M}}$ induces a restricted functional map $\mathcal{H}_{i}(\mathcal{M}) \to \mathcal{H}_{i}(\mathcal{N})$. The conformal term of the energy (Eq. \eqref{Eq:ConformalTerm}) can be easily evaluated on these subspaces. The ``conformal energy'' approach consists in choosing the shifts $\{\alpha_i\}_{i=1}^{k}$ (see Sec.~\ref{Sec:InitialGuess}) such that they minimize the conformal energy of the resulting $\mathcal{H}_{i}(\mathcal{M}) \to \mathcal{H}_{i}(\mathcal{N})$ map.\\

 Fig.~\ref{fig:method_init_and_fast_vs_principled} (left) depicts the performance of the three initializations in terms of geodesic error on the SHREC'20 dataset (lores), using $7$ landmarks. Tab.~\ref{tab:SHREC20_initialization_comparison} provides quantitative evaluations for the same experiment in terms of averaged geodesic error and Dirichlet energy. The ``normal derivatives'' approach slightly outperforms the other two on all metrics, which is why it is the one used in the main text.
 
 \begin{figure}
    \centering
    \includegraphics[width=0.45\linewidth]{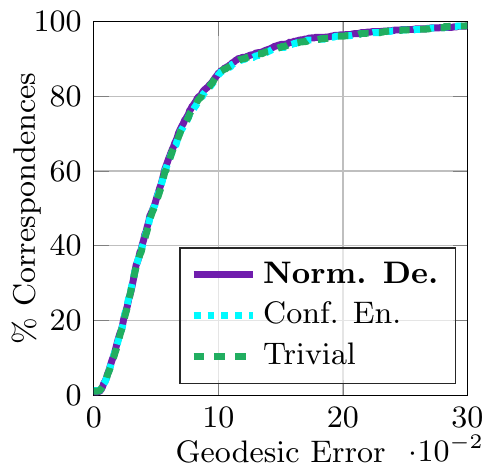}
    \includegraphics[width=0.45\linewidth]{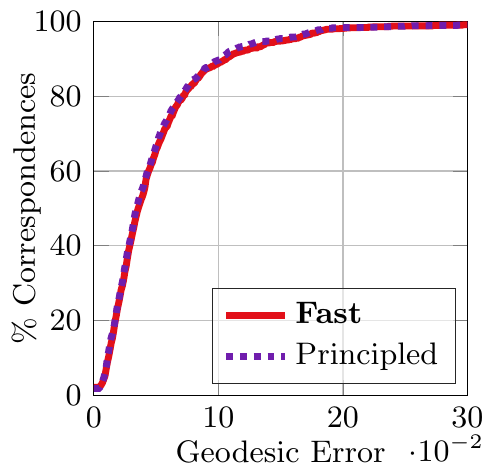}
    \caption{\label{fig:method_init_and_fast_vs_principled}\textbf{Left:} comparison of initializations for our method, where ``Norm. De.'' and ``Conf. En.'' respectively stand for ``Normal Derivative'' and ``Conformal Energy''. \textbf{Right:} comparison of the ``fast'' and ``principled'' energy formulations of our method. Both experiments are performed on the SHREC'20 lores dataset (partial shapes excluded).}
\end{figure}

\begin{table}
\centering
\begin{tabular}{cccc} 
\toprule
Method        & Av. Geo. Err. & Dir. E. & Av. Time (in s.) \\ 
\midrule
Trivial       & $6.36\times10^{-2}$     &   $16.8$    &   $41.4$   \\ 
\hline
Conf. En.      &    $6.36\times10^{-2}$     &    $16.7$   &   $53.2$    \\ 
\hline
\textbf{Norm. De.} &    $\mathbf{6.26\times10^{-2}}$    &    $\mathbf{16.2}$    &    $\mathbf{40.4}$   \\
\bottomrule
\end{tabular}
\caption{Quantitative evaluation results on the SHREC'20 lores (without partial shapes) data sets. The average geodesic error (Av. Geo. Err.), the Dirichlet energy (Dir. E.) and average execution time (Av. Time) on both data sets are displayed for the three initialization methods that we tried: Trivial, Conformal Energy (``Conf. En.'') and Normal Derivatives (``Norm. De.''). Normal Derivatives is the method used in the rest of the paper.}
\label{tab:SHREC20_initialization_comparison}
\end{table}

\subsection{Comparison of the ``Principled'' and ``Fast'' Energy Optimization}\label{Sec:appendix_principled_vs_fast}

At the end of Sec.~\ref{Sec:ConversionToNearestNeighbor}, we introduced an unprincipled way to accelerate the nearest neighbor search used in the solution of our problem. In this section, we quantitatively compare this ``fast'' method to the ``principled'' one on the SHREC'20 data set (partial shapes excluded). The output of this evaluation is displayed in Fig.~\ref{fig:method_init_and_fast_vs_principled} (right) and Tab.~\ref{tab:fast_principled_comparison}. While very similar in terms of matching performance, the ``fast'' method is more than three times faster to compute.  We therefore employ it instead of the ``principled'' approach. Note that the more than threefold speedup is consistent with the fact that the matrices used in the ``fast'' method are three times smaller.

\begin{table}
\centering
\begin{tabular}{ccc} 
\toprule
Method        & Av. Geo. Err. & Av. Time (in s.) \\ 
\midrule
Principled       & $\mathbf{4.96\times10^{-2}}$     &   $184$   \\ 
\hline
\textbf{Fast} &    $5.13\times10^{-2}$    &    $\mathbf{48.7}$   \\
\bottomrule
\end{tabular}
\caption{Average geodesic error (Av. Geo. Err.) and average execution time (Av. Time) associated to the comparison of the ``principled'' and ``fast'' computation methods.}
\label{tab:fast_principled_comparison}
\end{table}

\subsection{Complementary benchmark on SHREC'20 lores}
As a complement to our main evaluation on SHREC'20 lores, we conducted an evaluation using only $8$ pairs from the initial benchmark to compare against the method proposed in~\cite{schmidt2020inter} (InterSurf). InterSurf, WA, HyperOrb FMap ZO and our approach obtain a geodesic error (scaled by a factor $\times 100$) of respectively $11.9$, $5.41$, $5.99$, $8.69$ and $\textbf{5.2}$. The restricted number of shapes on which we evaluate is due to the fact that InterSurf does not handle shapes with complex topologies well. In particular, the method assumes that the meshes are watertight and share the same genus, in strong contrast to our approach that does not make such assumptions. However, we note that this method was not primarly designed for shape matching.

\subsection{Additional Qualitative Evaluations}\label{sec:appendix_additional_quali}
We provide additional qualitative evaluations on isometric and non-isometric shape pairs in order to show best- and worst-case shape matching scenarios for our method. 

For isometric shapes, the best pairs are depicted in Fig.~\ref{fig:best_iso_quali} and the worst pairs in Fig~\ref{fig:worst_iso_quali}. 

For non-isometric shapes, the best pairs are illustrated in Fig.~\ref{fig:best_noniso_quali} and the worst pairs in Fig.~\ref{fig:worst_noniso_quali}.

Finally, in Fig.~\ref{fig:bestWorst_SHREC19_quali}, we show the best and worst pairs for the SHREC'19 benchmark.

\begin{figure}[ht]
\centering
    \begin{overpic}[height=4.25cm, trim=-25 0 0 -5, clip]{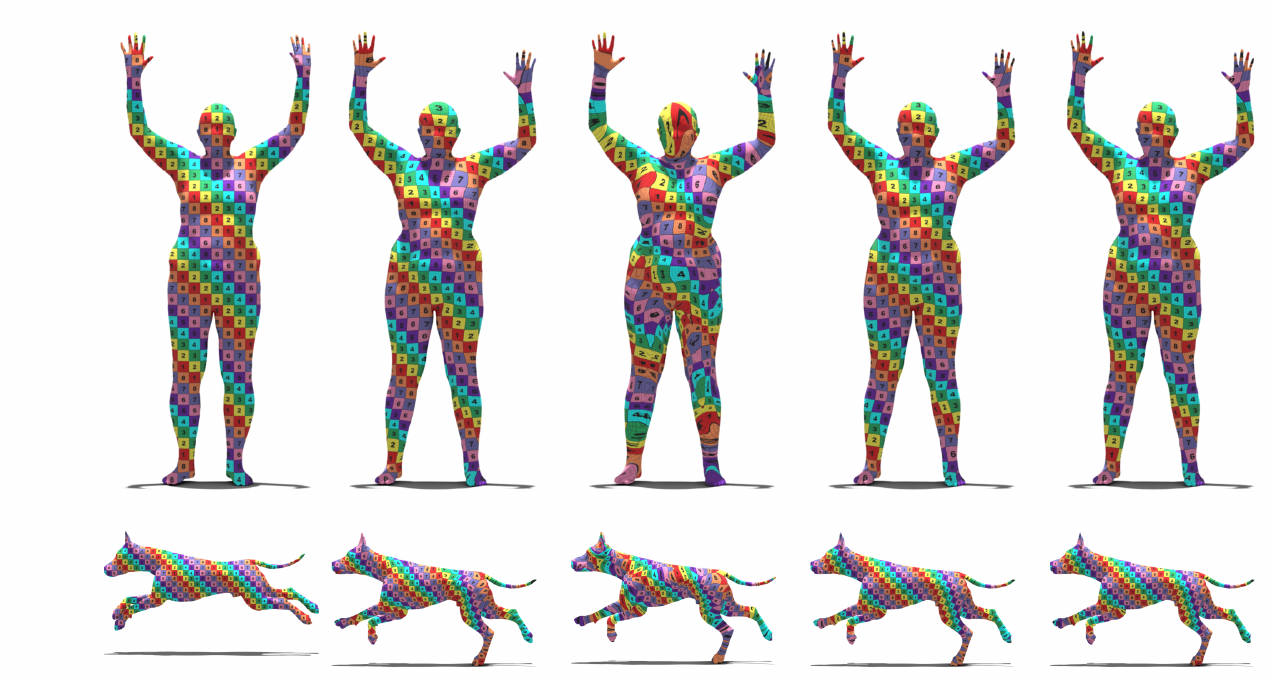}%
    \put(0,30) {FAUST}%
    \put(0,7) {TOSCA}%
    \put(16,49) {Source}%
    \put(31,49) {HyperOrb}%
    \put(53,49) {WA}%
    \put(66,49) {FMap ZO}%
    \put(88,49) {\textbf{Ours}}%
    \end{overpic}
\caption{\label{fig:best_iso_quali}Qualitative evaluation of our method and competitors on \textbf{isometric shapes} from the FAUST dataset (top row) and the TOSCA isometric dataset (bottom row). The shape pair is selected such that the geodesic error of our method is \textbf{the best} over the dataset.}
\end{figure}

\begin{figure}[ht]
\centering
    \begin{overpic}[height=5cm, trim=-20 0 0 -5, clip]{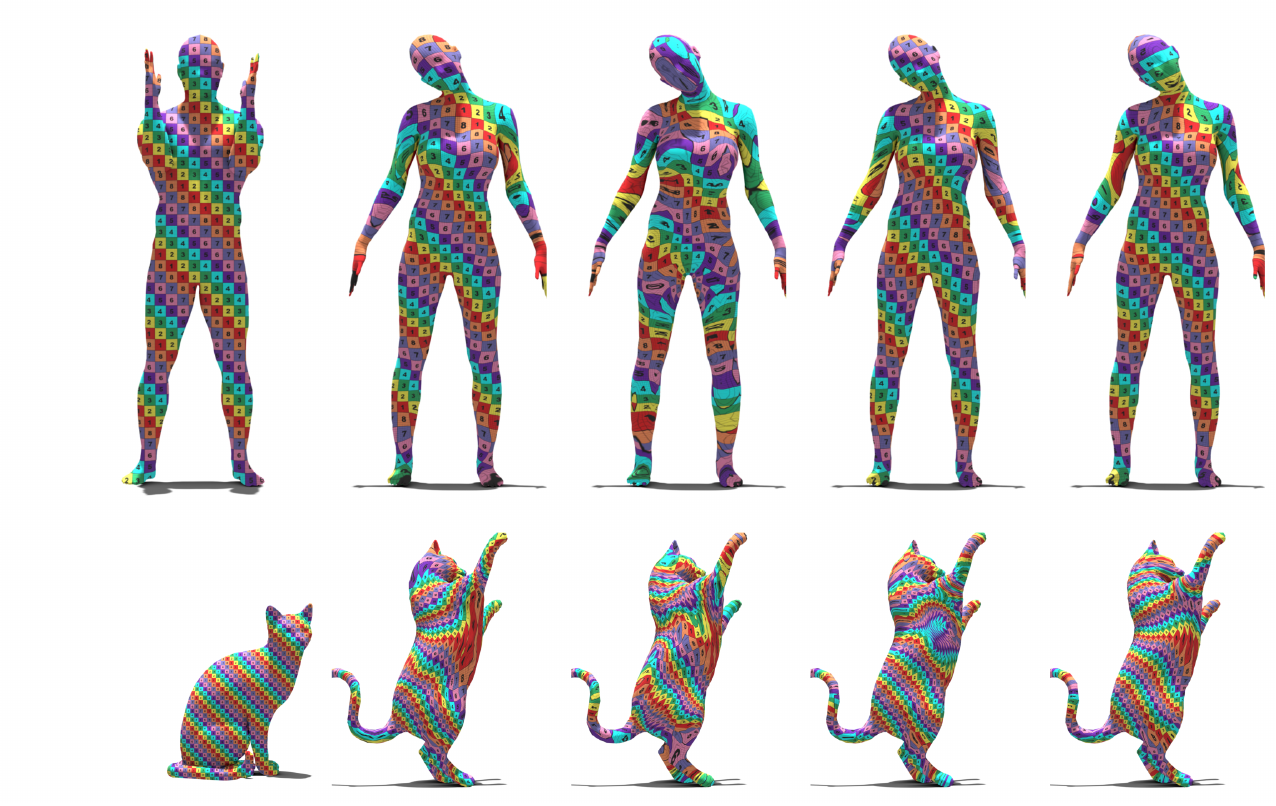}%
    \put(0,38) {FAUST}%
    \put(0,7) {TOSCA}%
    \put(13,59) {Source}%
    \put(29,59) {HyperOrb}%
    \put(53,59) {WA}%
    \put(65,59) {FMap ZO}%
    \put(88,59) {\textbf{Ours}}%
    \end{overpic}
\caption{\label{fig:worst_iso_quali}Qualitative evaluation of our method and competitors on \textbf{isometric shapes} from the FAUST dataset (top row) and the TOSCA isometric dataset (bottom row). The shape pair is selected such that the geodesic error of our method is \textbf{the worst} over the dataset.}
\end{figure}

\begin{figure}[ht]
\centering
    \begin{overpic}[height=3.5cm, trim=-45 10 0 -20, clip]{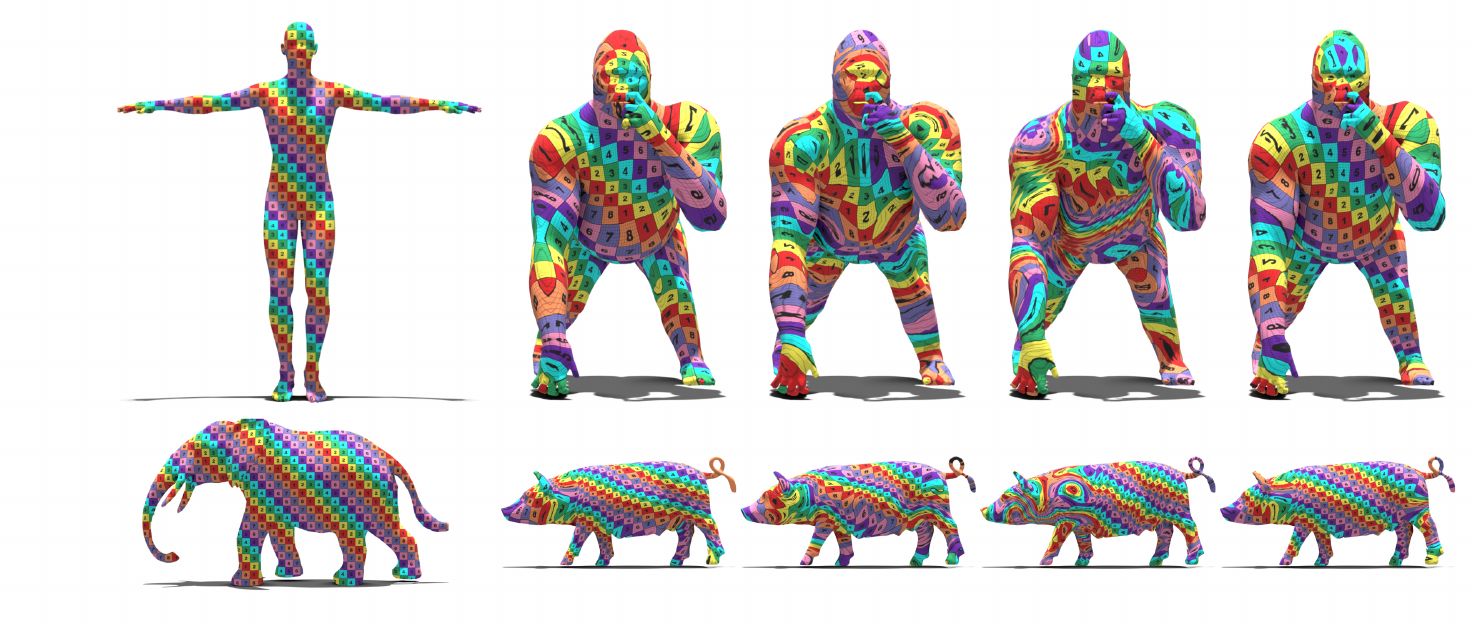}%
    \put(0,23) {TOSCA}%
    \put(0,5) {SHREC'20}%
    \put(22,37) {Source}%
    \put(39,37) {HyperOrb}%
    \put(58,37) {WA}%
    \put(68,37) {FMap ZO}%
    \put(87,37) {\textbf{Ours}}%
    \end{overpic}
\caption{\label{fig:best_noniso_quali}Qualitative evaluation of our method and competitors on \textbf{non-isometric shapes}. The first row corresponds to shapes from the TOSCA non-isometric data set. The bottom row consists of shapes from the SHREC'20 lores data set. The shape pair is selected such that the geodesic error of our method is \textbf{the best} over the dataset.}
\end{figure}

\begin{figure}[ht]
\centering
    \begin{overpic}[height=3.75cm, trim=-20 0 0 -10, clip]{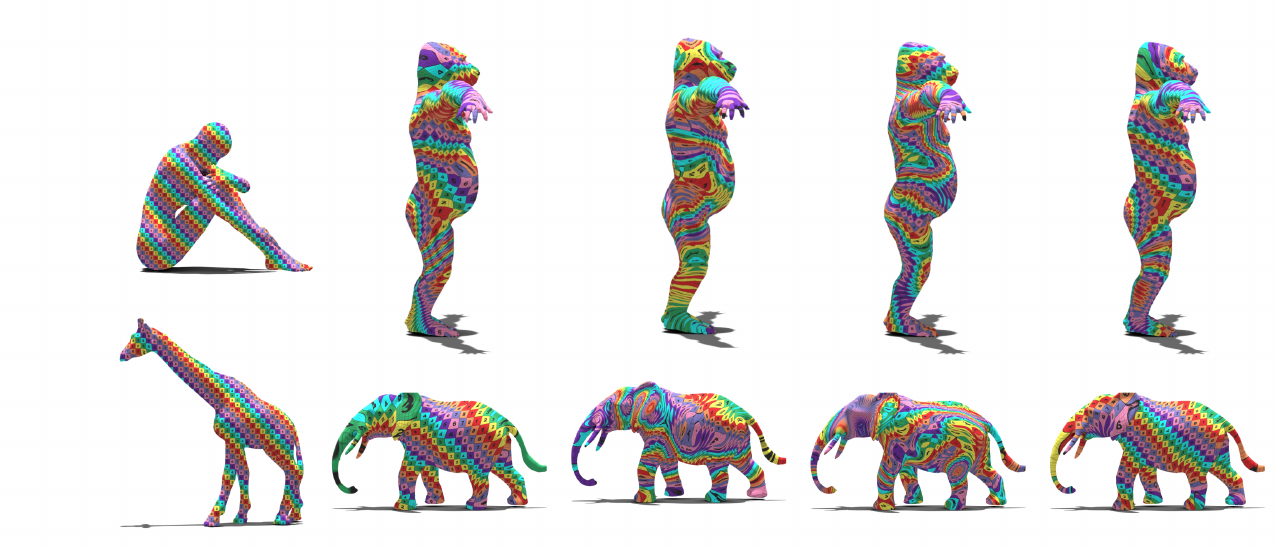}%
    \put(0,23) {TOSCA}%
    \put(0,5) {SHREC'20}%
    \put(13,40) {Source}%
    \put(28,40) {HyperOrb}%
    \put(53,40) {WA}%
    \put(65,40) {FMap ZO}%
    \put(87,40) {\textbf{Ours}}%
    \end{overpic}
\caption{\label{fig:worst_noniso_quali}Qualitative evaluation of our method and competitors on \textbf{non-isometric shapes}. The first row corresponds to shapes from the TOSCA non-isometric data set. The bottom row consists of shapes from the SHREC'20 lores data set. The shape pair is selected such that the geodesic error of our method is \textbf{the worst} over the dataset.}
\end{figure}

\begin{figure}[ht]
    \centering
    \begin{overpic}[height=4cm, trim=-30 0 0 -10, clip]{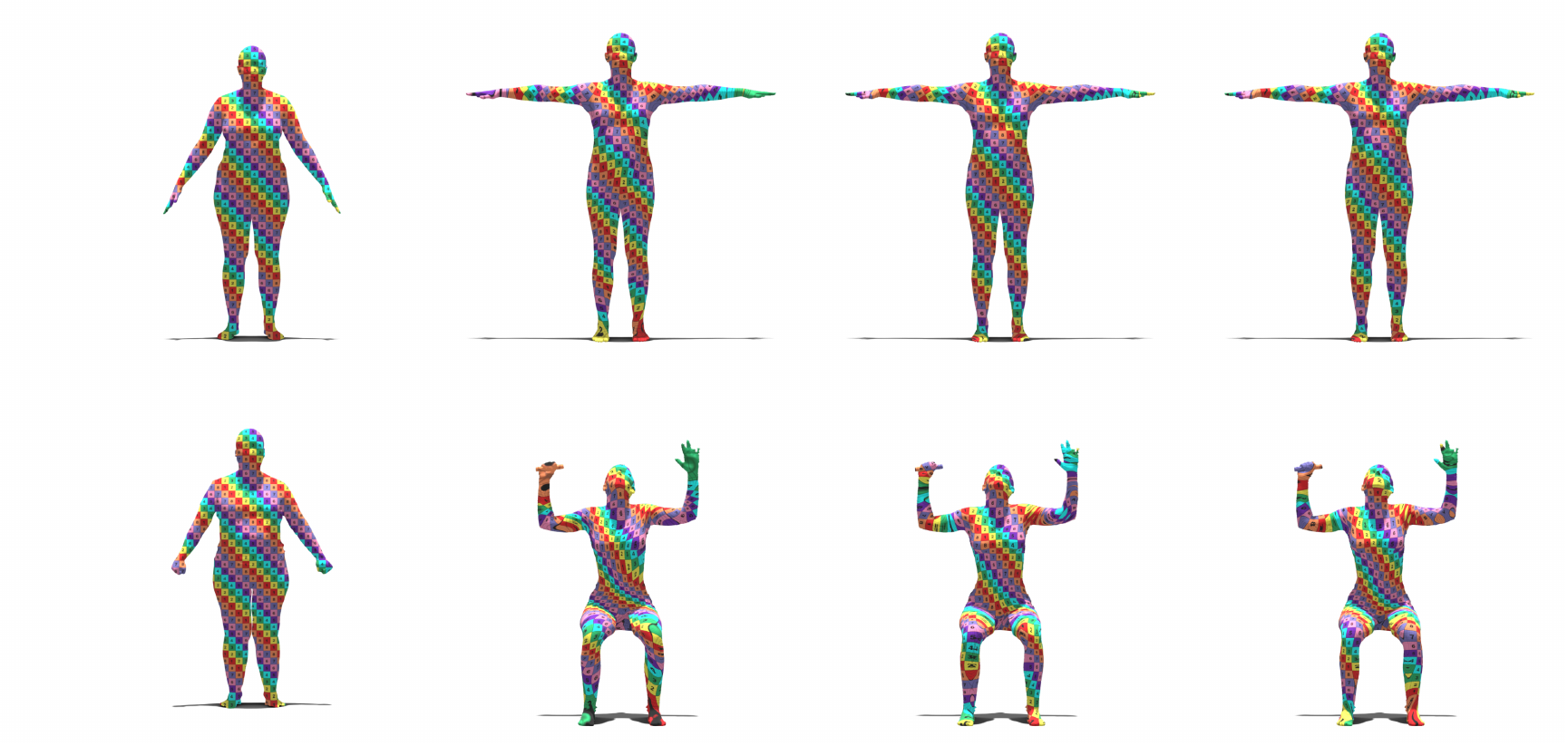}%
    \put(0,35) {Best}%
    \put(0,11) {Worst}%
    \put(16,44) {Source}%
    \put(34,44) {HyperOrb}%
    \put(58,44) {FMap ZO}%
    \put(84,44) {\textbf{Ours}}%
    \end{overpic}
    \caption{\label{fig:bestWorst_SHREC19_quali}Qualitative evaluation of our method and competitors on the \textbf{SHREC'19} data set. The first row corresponds to the \textbf{best} shape pair, while the bottom row corresponds to the \textbf{worst} shape pair on this data set.}
    
\end{figure}

\section{Additional Parameter Study} \label{Sec:appendix_additional_parameter_study}

\subsection{Study of the Weights in the Energy} \label{Sec:appendix_energy_weights}
We define three weights to compute a point-to-point map between two shapes based on the energy (Eq.~\eqref{Eq:MinimizationProblem}): the conformal, the properness and the invertibility weights, denoted respectively $a_C$, $a_P$ and $a_I$. Since we normalize the weights, their absolute value is unimportant.

To study how their relative value influences the quality of the output map we conduct a dedicated experiment on the SHREC'20 dataset, with shapes remeshed to count $1$K vertices and excluding partial shapes. $8$ landmarks in ground-truth correspondence are placed on each shape, in the locations described in App.~\ref{Sec:appendix_evaluation_setup_details}. For each set of weight values, the geodesic error and the Dirichlet energy (see App.~\ref{Sec:appendix_dirichlet_energy}), averaged over all shape pairs (in both directions) in the dataset, are computed.

We first fix the conformality weight to $1$ and vary the two remaining weights within a range of energy values in Fig.~\ref{fig:SHREC20_1k_weights_study} left (geodesic error) and Fig.~\ref{fig:SHREC20_1k_weights_study_DIRICHLET} left (Dirichlet energy). Second, we let one weight vary and fix the two remaining values either to $0$ or to $1$, as illustrated in Fig.~\ref{fig:SHREC20_1k_weights_study} right (geodesic error) and Fig.~\ref{fig:SHREC20_1k_weights_study_DIRICHLET} right (Dirichlet energy). Finally, we report in Tab.~\ref{tab:SHREC20_1k_weights_study_2_at_0} the average geodesic error and Dirichlet energy on the data set, obtained when fixing one weight to $1$ and setting the two others to $0$. This experiment allows to measure which term carries the greatest influence on the final map quality.

These quantitative evaluations highlight the existence of a trade-off between the accuracy of the map (minimization of the geodesic error) and the smoothness of the map (minimizing the Dirichlet energy) when choosing the weight configuration. Roughly speaking, the invertibility and properness terms promote accuracy, while the conformality term promotes smoothness.

Since this trade-off is application-dependent, we leave the fine-tuning of the energy weights to the end-user and set all weights to $1$ in the remaining of our experiments as it provides a satisfactory balance in practice.

\begin{figure}[t]
    \centering
    \includegraphics[height=2cm]{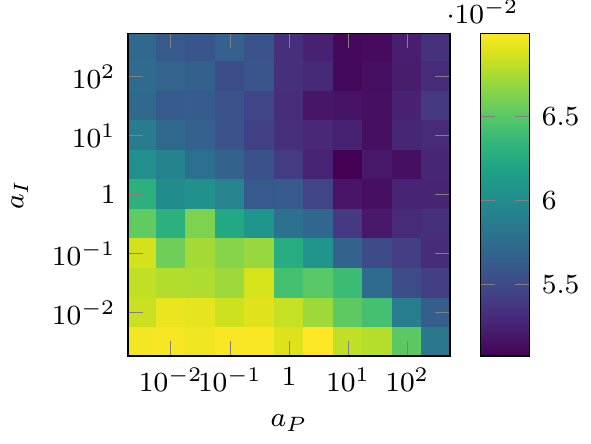}
    \includegraphics[height=2cm]{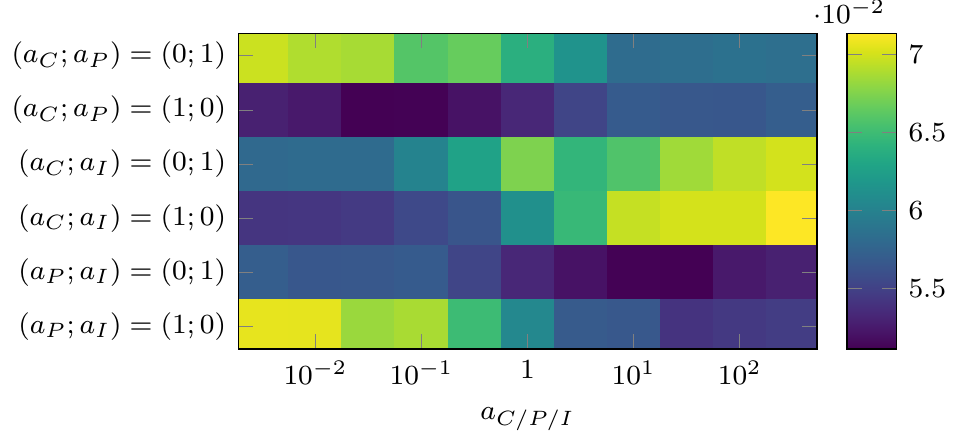}
\caption{\label{fig:SHREC20_1k_weights_study} Weight study on the SHREC'20 data set (full shapes remeshed to $1$K vertices). The error measure is the \textbf{mean geodesic error}, averaged on the data set. $a_C$, $a_P$ and $a_I$ are the Conformality, Properness and Invertibility weights. On the \textbf{left}, we fix the conformality weight $a_C$ and vary the properness and invertibility weights $a_P$ and $a_I$. On the \textbf{right}, we vary one weight $a_{C/P/I}$ and fix the remaining weights either to $0$ or to $1$.}
\end{figure}

\begin{figure}[t]
    \centering
    \includegraphics[height=2cm]{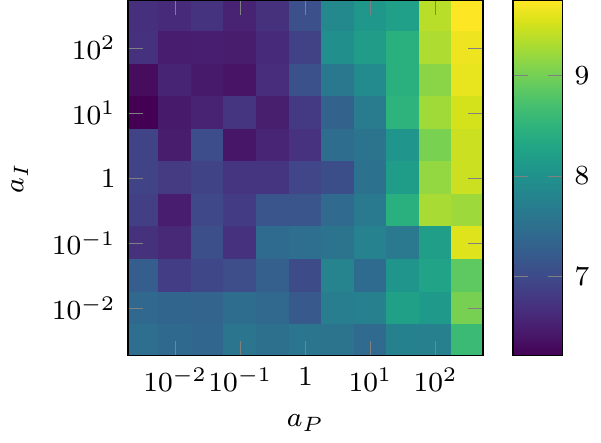}
    \includegraphics[height=2cm]{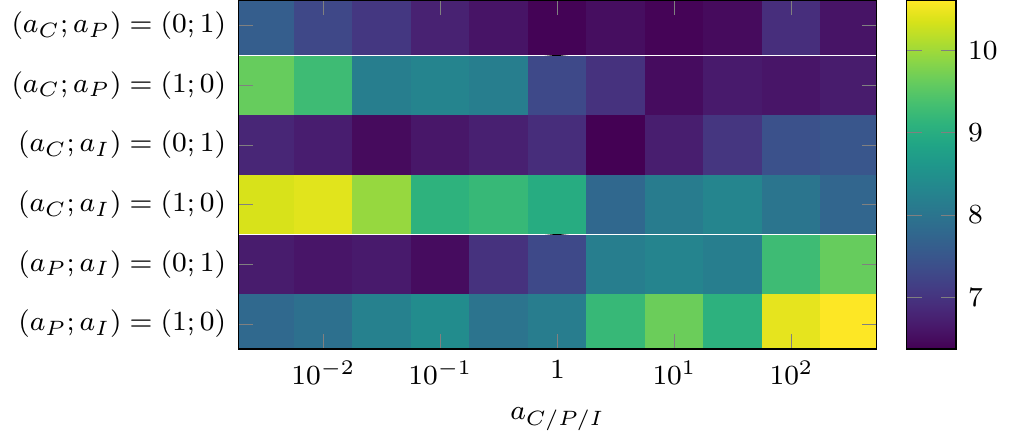}
\caption{\label{fig:SHREC20_1k_weights_study_DIRICHLET} Weight study on the SHREC'20 data set (full shapes remeshed to $1$K vertices). The error measure is the \textbf{Dirichlet energy}, averaged on the data set. $a_C$, $a_P$ and $a_I$ are the Conformality, Properness and Invertibility weights. On the \textbf{left}, we fix the conformality weight $a_C$ and vary the properness and invertibility weights $a_P$ and $a_I$. On the \textbf{right}, we vary one weight $a_{C/P/I}$ and fix the remaining weights either to $0$ or to $1$.}
\end{figure}

\begin{table}
\centering
\begin{tabular}{cccc}
\toprule
Non-Zero Weight                         & Av. Geo. Err. & Dir. E.\\
\midrule
Conformality ($a_{C}$)      &    $5.91\times10^{-2}$     &  $6.82$  \\
\hline
Properness ($a_{P}$)     &    $7.06\times10^{-2}$    &  $7.82$  \\
\hline
Invertibility ($a_{I}$)      &    $5.42\times10^{-2}$     &   $11.4$ \\
\bottomrule
\end{tabular}
\caption{Quantitative evaluation results on the SHREC'20 data set (full shapes remeshed to $1$K vertices) when fixing one weight to $1$ (Non-Zero Weight) and setting the remaining weights to $0$. The average geodesic error (Av. Geo. Err.) and Dirichlet Energy (Dir. E.) is given for each.}
\label{tab:SHREC20_1k_weights_study_2_at_0}
\end{table}

\subsection{Landmark Sampling Qualitative Illustration}\label{Sec:appendix_other_landmark_experiments}
We visualize qualitatively the interest of introducing more landmark correspondences in Fig.~\ref{fig:extended_teaser}. In this visualisation, since ``HyperOrb'' does not support less than $5$ landmark correspondences, no map for $3$ and $4$ landmark correspondences can be computed for this method. 

Note how the regions around the mouth and the eyes are accurately mapped with our approach compared to the two other approaches.

\begin{figure}
    \centering
    \includegraphics[width=\columnwidth]{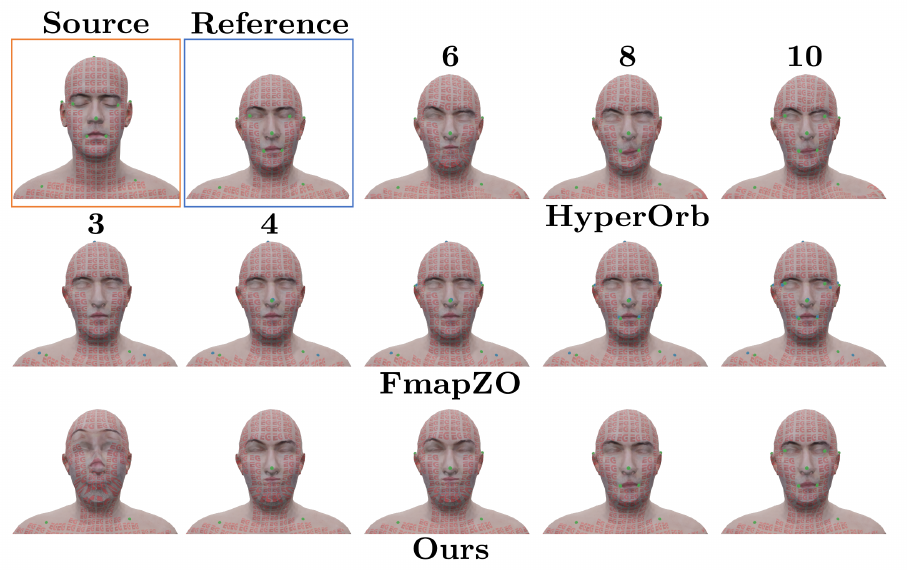}
    \caption{Qualitative comparison of our method to competitors when increasing the number of landmarks on the same shape pair as for our teaser (Fig.~\ref{fig:teaser}). The ground truth landmark locations are denoted by green dots. In the case of FMapZO (no exact landmark preservation), the blue dots indicate the location of the mapped landmarks.}
    \label{fig:extended_teaser}
\end{figure}

\subsection{Basis Near-Orthogonality}\label{Sec:appendix_near_orthogonality}
For each shape $\mathcal{M}$ of the SHREC'19 data set~\cite{SHREC19}, we compute the matrix with entries  $m_{i,j}=\left| \langle\Phi^{_{\mathcal{M}}}_i,\Phi^{_{\mathcal{M}}}_j \rangle_{_{W(\mathcal{M})}}\right|$, where  $\Phi^{_{\mathcal{M}}}_i$ designates the $i$-th basis vector. We use $7$ landmarks, $10$ Dirichlet-Steklov eigenfunctions, leading to a Dirichlet-Steklov block of size $70\times70$, and $120$ Dirichlet Laplacian eigenfunctions. Since we are only interested in the computation of the basis itself in this setup, the landmarks were placed at random locations to maximize the diversity of situations encountered. The average of all matrices is displayed in Fig.~\ref{fig:average_dot_product_matrix}. Note the clear diagonal behavior, that is in agreement with our observations on a simple sphere shape (Fig.~\ref{Fig:SphereWproducts}).

\begin{figure}
    \centering
    \includegraphics[width=6cm]{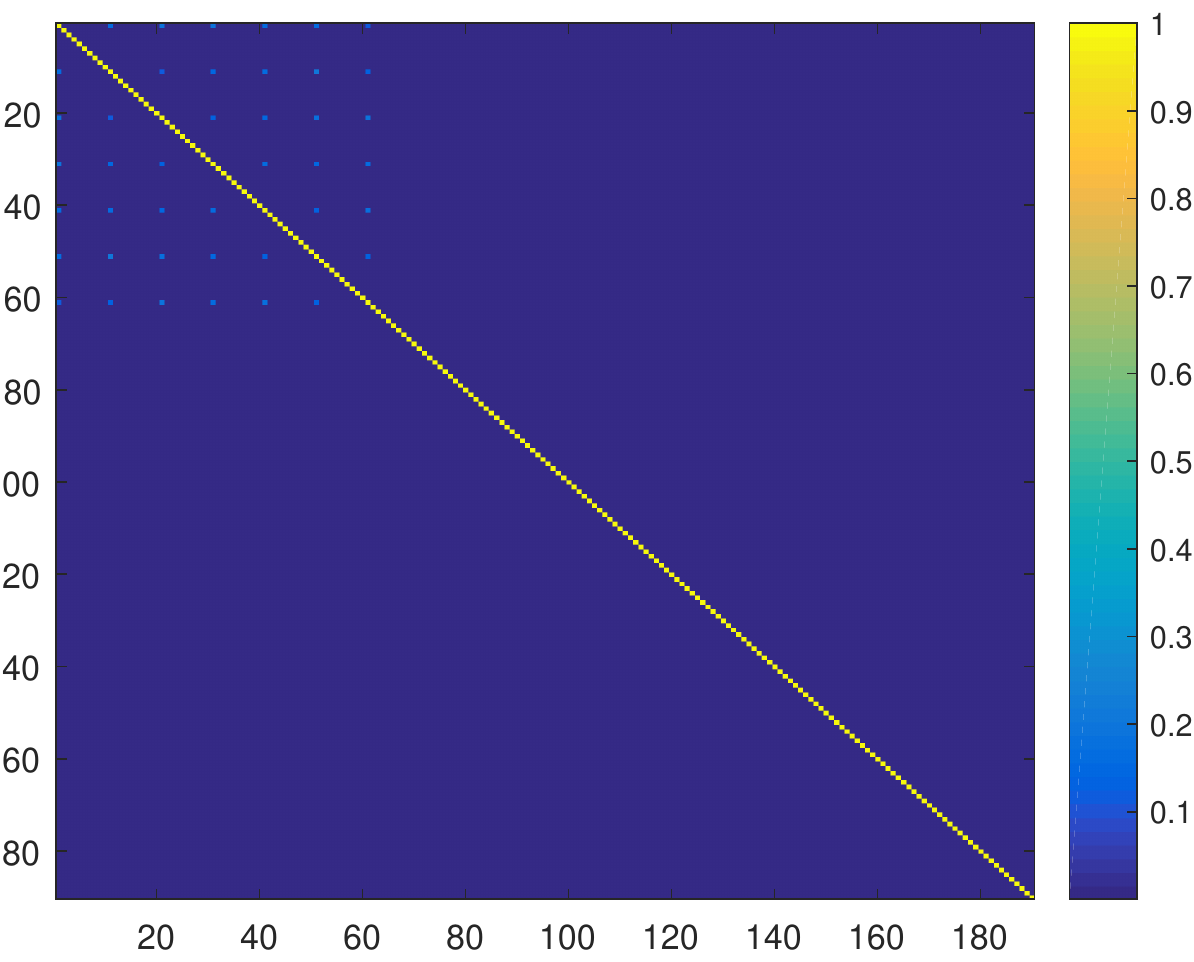}
    \caption{Average of the absolute values of the inner product matrix of each shape in the SHREC'19 data set. Except for the first few Dirichlet-Steklov eigenfunctions, the off-diagonal inner products are negligible. This validates the approximation of orthogonality. We highlight that this computation also sheds light on the robustness of our basis computation to complex triangulation and partiality setups.}
    \label{fig:average_dot_product_matrix}
\end{figure}

\subsection{Number of Basis Functions}\label{Sec:appendix_num_LB_DS_eigs}
To select the number of basis functions for $\mathcal{G}(\mathcal{M})$ and each $\mathcal{H}_{j}(\mathcal{M})$ (see Sec.~\ref{Sec:LandmarkAdaptedBasis}), we study their respective size $N_{\LB}$ and $N_{\DS}$ separately, as illustrated in Fig.~\ref{fig:basis_size}.

Increasing the size of $\mathcal{G}(\mathcal{M})$ slightly increases the matching performance up to
$N_{\LB}=120$. In contrast, varying $N_{\DS}$ above $10$ decreases
the quality of the maps. Hence, we fix the following basis sizes throughout the rest of the article: $N_{\LB}=120$ and $N_{\DS}=10$.

\begin{figure}
    \centering
    \includegraphics[width=0.45\linewidth]{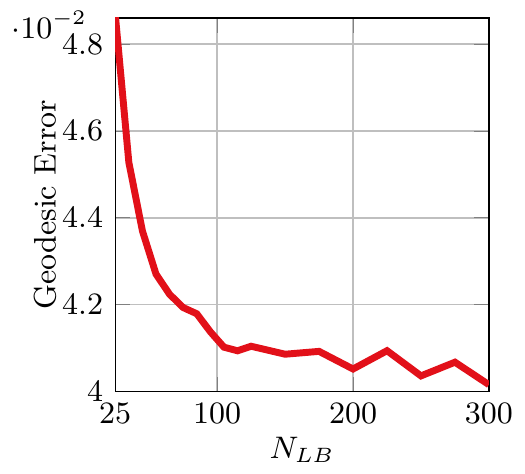}
    \includegraphics[width=0.45\linewidth]{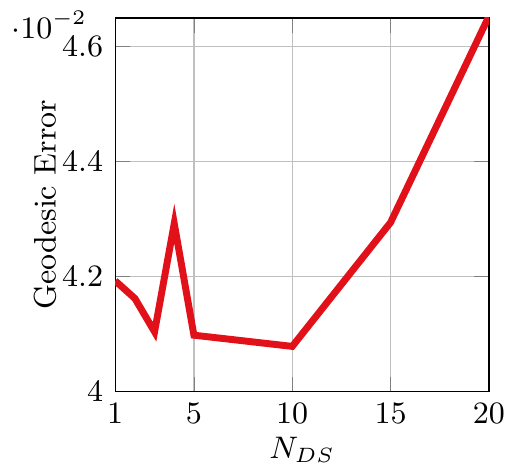}
    \caption{\label{fig:basis_size}Effect of varying the size of our basis on the $\mathcal{G}(\mathcal{M})$ space (\textbf{left}) and on the $\mathcal{H}_{j}(\mathcal{M})$ space (\textbf{right}). Both figures are an average over all pairs of the TOSCA non-isometric dataset.}
\end{figure}

\section{Evaluation Setup Details}\label{Sec:appendix_evaluation_setup_details}
\subsection{Landmark Position}
The benchmark datasets that we use contain either humanoid shapes (humans and gorillas) or four-legged animals. Depending on the type of creature, we place our landmarks at either $7$ or $8$ semantically compatible locations:

\begin{enumerate}
    \item Top of the head
    \item Bottom of the right (hind) leg
    \item Bottom of the left (hind) leg
    \item Bottom of the right front leg / extremity of the third finger on the right hand
    \item Bottom of the left front leg / extremity of the third finger on the left hand
    \item Middle of the belly/umbilicus
    \item Middle of the back
    \item Tip of the tail (Four-legged animals only)
\end{enumerate}
The last landmark is only used on the TOSCA and SHREC'20 data sets. Notice that our landmark placement is reminiscent of farthest point sampling. The landmark placement is common to all considered methods. The other parameters depend on the method used.

\subsection{Method Configuration}
\textbf{Hyperbolic Orbifold Tutte Embeddings (hyperOrb) and Weighted Averages (WA).} These methods do not require any additional parameters.
\par
\textbf{Functional Maps With ZoomOut Refinement (FMap ZO).} A $20\times20$ functional map is computed for each source-target pair in setup 1 and 2, following the setup of~\cite{melzi2019zoomout}.
In particular, we use wave kernel signature and wave kernel map functions as descriptors. The descriptor functions are computed at the same ground truth landmark positions used for the other methods. At each landmark location, 12 wave kernel map functions are computed using a basis of 120 LB-eigenfunctions.

The energy employed to compute the functional map leverages the descriptor preservation, descriptor commutativity and LB-commutativity terms. Contrary to~\cite{melzi2019zoomout}, we did not employ the orientation term in the energy. Indeed, with a high number of landmarks as in our setup, the symmetry ambiguities are easily solved by the functional maps pipeline.
\par
\textbf{Ours.} We use the provided landmark locations together with the settings specified previously. We summarize them here for convenience.
\begin{itemize}
    \item Energy weights: $a_C = a_P = a_I = 1$.
    \item Number of Dirichlet-Steklov eigenfunctions per landmark: $N_{\DS} =10$.
    \item Number of Dirichlet Laplacian eigenfunctions: $N_{\LB} =120$.
    \item Landmark circle size factor: $r_f = 0.5$.
\end{itemize}

\noindent Moreover, recall that we use the acceleration strategy described at the end of Sec.~\ref{Sec:ConversionToNearestNeighbor}.